\documentclass[12pt,reqno]{amsart}
\usepackage{amsaddr}

\usepackage{amsmath,amssymb,amsthm}
\usepackage{bm}
\usepackage{natbib}
\usepackage[dvipdfmx]{graphicx}
\usepackage{here,comment}
\usepackage{algorithm,algorithmic}
\usepackage{xcolor}
\usepackage{booktabs}
\usepackage{amsmath}
\usepackage{mathtools}
\usepackage{enumerate}
\usepackage[colorlinks=true, allcolors=blue]{hyperref}
\usepackage{fullpage}

\usepackage{tikz}
\usepackage{tcolorbox}
\tcbuselibrary{breakable, skins}

\newtheorem{theorem}{Theorem}
\newtheorem{corollary}[theorem]{Corollary}
\newtheorem{lemma}[theorem]{Lemma}
\newtheorem{proposition}[theorem]{Proposition}
\newtheorem{assumption}{Assumption}
\theoremstyle{definition}
\newtheorem{definition}{Definition}

\newtheorem*{ansatz}{Ansatz}

\newcommand{\R}{\mathbb{R}}
\newcommand{\N}{\mathbb{N}}

\newcommand{\mL}{\mathcal{L}}

\newcommand{\mN}{\mathcal{N}}

\newcommand{\Ep}{\mathbb{E}}

\renewcommand{\hat}{\widehat}
\renewcommand{\tilde}{\widetilde}

\DeclareMathOperator*{\plim}{p-lim}

\usepackage{color}

\usepackage{newtxtext,newtxmath}

\usepackage{mathtools}
\mathtoolsset{showonlyrefs=true}

\allowdisplaybreaks

\author{Shota Imai$^1$, Sota Nishiyama$^{1,2}$,  Masaaki Imaizumi$^{1,2}$}

\address{$^1$The University of Tokyo, Bunkyo, Tokyo, Japan.\\
$^2$RIKEN Center for Advanced Intelligence Project, Chuo, Tokyo, Japan}

\date{ \today}

  \title{Dichotomy of Feature Learning and Unlearning:\\ Fast-Slow Analysis on Neural Networks with Stochastic Gradient Descent}
  
\begin{document}

\maketitle

\begin{abstract}
 
The dynamics of gradient-based training in neural networks often exhibit nontrivial structures; hence, understanding them remains a central challenge in theoretical machine learning. In particular, a concept of \emph{feature \textbf{un}learning}, in which a neural network progressively loses previously learned features over long training, has gained attention. In this study, we consider the infinite-width limit of a two-layer neural network updated with a large-batch stochastic gradient, then derive differential equations with different time scales, revealing the mechanism and conditions for feature unlearning to occur. Specifically, we utilize the \textit{fast-slow dynamics}: while an alignment of first-layer weights develops rapidly, the second-layer weights develop slowly. The direction of a flow on a critical manifold, determined by the slow dynamics, decides whether feature unlearning occurs. We give numerical validation of the result, and derive theoretical grounding and scaling laws of the feature unlearning. Our results yield the following insights: (i) the strength of the primary nonlinear term in data induces the feature unlearning, and (ii) an initial scale of the second-layer weights mitigates the feature unlearning. Technically, our analysis utilizes Tensor Programs and the singular perturbation theory.

\end{abstract}

\section{Introduction}

\textbf{Background.}
Understanding the dynamics of gradient-based training in neural networks is a
central problem in modern machine learning.
Beyond static characterizations such as loss landscapes or stationary points,
it has become increasingly clear that many learning phenomena are inherently
\emph{dynamical}.
Especially, in high-dimensional regimes, self-averaging often enables a drastic simplification: 
the learning dynamics can be described by a small number of macroscopic order parameters.
Several theoretical frameworks make this reduction precise, i.e., the dynamical mean-field theory
\citep{bordelon2022self,celentano2021high}, the Tensor Programs
\citep{yang2019scaling,yang2019tensor1,yang2020tensor2,yang2021tensor2b}, and the generalized first-order method \citep{celentano2020estimation}.
Research on the learning dynamics of high-dimensional neural networks is rapidly advancing.

Key discoveries from analyzing dynamics include \textit{feature learning}, which refers to the process where shallow layers of neural networks learn the feature structures of data-generating models, explaining why multi-layer structures achieve better accuracy. These were analyzed in \citet{ba2022high, damian2022neural, moniri2024theory, yang2021tensor4}, demonstrating that neural networks trained with appropriate design can avoid the so-called lazy regime and achieve feature learning. 

In contrast, \textit{feature \textbf{un}learning} has been proposed as an important notion related to feature learning. 
Feature unlearning refers to the phenomenon where shallow layers of neural networks forget feature structures they have previously learned, and could serve as one theory explaining the mechanism of deep learning.
Particularly, \citet{montanari2025dynamical} studies a neural network updated by a gradient flow and identifies a pronounced separation of time scales in two-layer networks, together with regimes in which previously
learned features are progressively forgotten.
These results suggest that feature
unlearning is not a pathological effect, but rather can be understood as a generic consequence of multiple time scales in high-dimensional training regimes. 
However, research on these important concepts is still in its infancy, since the analytical framework is currently limited to updates via gradient flow, hence its underlying \emph{mechanism} remains incompletely understood.

\begin{figure*}[t]
    \centering
    \includegraphics[width=0.9\linewidth]{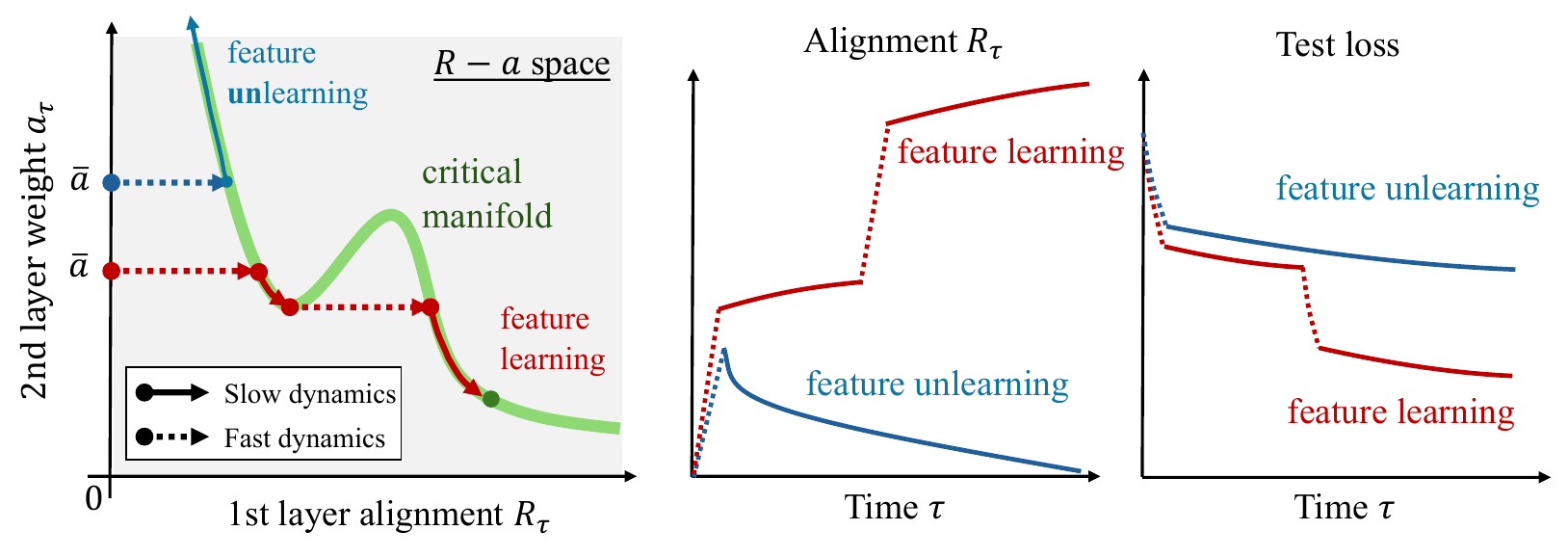}
    \caption{Fast-slow dynamics of first-layer alignment ($R_\tau$) and second-layer weights ($a_\tau$) in time $\tau$, explaining the evolution of alignment and test loss. In the space of $R_\tau$ and $a_\tau$, the $R-a$ space in the left panel, we find a \textit{critical manifold} (green curve). Each trajectory starts from its initial point $(0,\bar{a})$ and, after reaching the manifold, slowly evolves along it. On all panels, the red trajectory represents feature learning, while the blue trajectory represents feature unlearning. The solid line represents slow dynamics, and the dashed line represents fast dynamics. In the feature learning case, test loss decreases like a staircase when $|R_\tau|$ increases via the fast dynamic. In the feature unlearning case, $|R_\tau|$ initially increases, then converges to zero due to the slow dynamics on the manifold.
    }
    \label{fig:feature_outline}
\end{figure*}

\textbf{Motivation.}
The purpose of this study is to investigate whether feature unlearning occurs in a more general neural network setting, namely, updates via stochastic gradient descent (SGD) in discrete time, and to clarify whether it also occurs there. Furthermore, it aims to reveal more rigorously the underlying principles of time scales by which feature unlearning occurs.

\textbf{Approach.}
To this aim, we consider a neural network updated by one-pass SGD and derive a critical equation representing its dynamics.
Starting from the setup with data generated from a single-index teacher model, 
we use the Tensor Programs and infinite-width limit, then derive a deterministic continuous-time differential equation to describe macroscopic variables of a neural network.

We, then, introduce an ansatz that reveals a separation of time-scales of variables of the derived system, casting the dynamics into a singularly perturbed problem.
Although the ansatz is not assumed a priori at the level of SGD, numerical experiments show that it provides an accurate description of the observed dynamics.
This reformulation allows us to derive a system to represent the long-time behavior of the macroscopic variables of the neural network and to isolate the reduced dynamics.

\textbf{Results and implications.}
Our analysis shows that feature unlearning arises as a direct consequence of
the reduced slow dynamics along a \textit{critical manifold} induced from the derived system in the space of macroscopic variables.
Specifically, in the infinite-width limit, trajectories of the macroscopic variables rapidly collapse onto the critical manifold and subsequently drift along it over much longer time scales, leading to a decay of feature alignment under explicit and verifiable conditions.
We identify the structure of the reduced slow dynamics responsible for the onset of feature unlearning and derive the associated asymptotic scaling laws, thereby providing a concrete and quantitative dynamical mechanism for the phenomenon.
Figure \ref{fig:feature_outline} outlines the mechanism.

\textbf{Key novelty and approach.}
The main contributions of this work are summarized as follows:
\begin{itemize}
\item \emph{From discrete SGD to macroscopic dynamics}:
Using the Tensor Programs framework, we derive a closed low-dimensional
representation for online SGD and obtain, in the large-width limit, a deterministic
ordinary differential equation (ODE) as a natural limit of the discrete algorithm.

\item \emph{Emergent fast-slow structure}:
Numerical simulations of the limiting ODE reveal a clear separation of time
scales, with fast convergence to a low-dimensional attracting set followed by
slow evolution, referred to the fast-slow decomposition, justifying a singular-perturbation description.
Based on the observation, we develop a new system for the fast-slow and its theoretical analysis.

\item \emph{Feature unlearning as slow dynamics on the manifold}:
Under the fast-slow structure, we show that feature unlearning arises from the slow dynamics along the attracting critical manifold.
Additionally, the staircase dynamics of test loss is described.
Using the singular perturbation theory, we characterize the conditions for
unlearning and derive its asymptotic scaling law.
\end{itemize}

\subsection{Notation}
Throughout the paper, $\|\cdot\|_2$ denotes the Euclidean norm.
For a differentiable function $f:\mathbb{R}\to\mathbb{R}$, we write $f'$ and
$f''$ for its first and second derivatives, respectively.
For a multivariate function, $\nabla$ denotes the gradient.
A function $f:\mathbb{R}\to\mathbb{R}$ is said to be \emph{polynomially bounded}
if there exist constants $C>0$ and $k\ge0$ such that
$
|f(x)| \le C(1+|x|^k), \text{for all } x\in\mathbb{R}.
$
We use the standard Landau notations $O(\cdot)$, $o(\cdot)$, and $\Theta(\cdot)$ with their usual meanings.
All random variables are defined on a common probability space, and
convergence in probability is denoted by $\plim$.

\section{Setup}

\subsection{Online supervised learning}
We study the supervised learning problem with an online learning setup.
Suppose that there exists a random variable $(y, \bm{x}) \in \R \times \R^d$ with a dimension $d \in \N$, which is characterized by the following model, referred to as a teacher model, with a function $f_\star:\R^d \to \R$ as
\begin{align} \label{eq:model}
    \bm{x} \sim \mN(0,I_d), \mbox{~and~}y = f_\star(\bm{x}) + \varepsilon,
\end{align}
where $\varepsilon \sim \mN(0,\sigma_\varepsilon^2)$ is an independent noise variable with its variance $\sigma_\varepsilon^2 > 0$.
Here, $\bm{x}$ is a feature vector and $y$ is a response variable.
The specific form of $f_\star$ will be formulated later. We assume that we observe data generated from this model \eqref{eq:model} under the online learning setup described below.

We consider a two-layer neural network model as a model to be trained, referred to as a student model.
Let $m \in \N$ be a width of the neural network and \( \bm{W} = (\bm{w}_1,\dots, \bm{w}_m) \in \mathbb{R}^{d \times m} \) and  \( \bm{a} = (a_1,\dots, a_m)^\top \in \mathbb{R}^m \) be the first and second layer weights, respectively.
Then, we study the neural network $f(\cdot\ ; \bm{a}, \bm{W}):\mathbb{R}^d \to \mathbb{R}$ with an input $\bm{x} \in \mathbb{R}^d$ as
\begin{align} \label{eq:neural_network}
    f(\bm{x}; \bm{a}, \bm{W}) = \frac{1}{m}\sum_{i=1}^m a_i \sigma(\langle \bm{w}_i, \bm{x} \rangle / \sqrt{d}),
\end{align}

We train the neural network model by the online learning.
In this setup, for each time $t \in \N$,  we observe a set of $n$ pairs of responses and feature vectors $\{ (y_i^t, \bm{x}_i^t) \}_{i=1}^n$, which is independent copy of $(y,\bm{x})$ from the model \eqref{eq:model}. We call the set as a batch and $n$ as a batch size. 
With the batch at time $t \in \N$, we define an empirical quadratic loss as 
\begin{align}
    \mL_t(\bm{a},\bm{W})= \frac{1}{2 n}\sum_{i=1}^n \big( y_i^t - f(\bm{x}_i^t; \bm{a}, \bm{W}) \big)^2.
\end{align}
Then, we conduct one-pass stochastic gradient descent (SGD) to update the parameters of the neural network model. Specifically, with an initialization $\bm{a}^0$ and $\bm{W}^0$, the one-pass SGD generates sequences $\bm{a}^1,\bm{a}^2,...$ and $\bm{W}^1,\bm{W}^2,...$ by the following recursive form:
\begin{align} \label{def:sgd_w}
    \bm{w}_i^{t + 1} &= \frac{\sqrt{d}}{\| \tilde{\bm{w}}_i^{t+1} \|_2}\tilde{\bm{w}}_i^{t+1} ,\quad \tilde{\bm{w}}_i^{t+1} = \bm{w}_i^t - \gamma d \nabla_{\bm{w}_i} \mathcal{L}_t(\bm{a}^t, \bm{W}^t), 
\end{align}
for $i=1,\dots, m$ and 
\begin{align} \label{def:sgd_a}
    \bm{a}^{t+1} &= \bm{a}^t - \gamma  \nabla_{\bm{a}} \mathcal{L}_t(\bm{a}^t, \bm{W}^t),
\end{align}
where $\gamma>0$ is a fixed learning rate.
Here, we added a normalization step on the first layer updates to guarantee that each $\bm{w}_i^t$ satisfies $\|\bm{w}_i^t\|_2 = \sqrt{d}\ (i=1,\dots,m)$ throughout training. Following \citet{montanari2025dynamical}, for theoretical convenience, here we introduced normalizing steps on the first layer updates \eqref{def:sgd_w}.

\subsection{Conditions}

Our theoretical analysis relies on several conditions. 
First, we provide the asymptotic settings for the batch size and the data dimension. This regime is common in theoretical analysis of neural networks, e.g., \citet{celentano2020estimation, celentano2021high, dandi2024two}.
\begin{assumption}[Proportionally high-dimension regime]\label{ass:high_dimension}
Both the batch size \( n \) and the feature dimension \( d \) diverge to infinity while preserving ${n}/{d} \to \delta$ with some $\delta \in (0,\infty)$.
\end{assumption}

\subsubsection{Teacher model}

We assume that the teacher model has the form of the single-index model. In particular, we introduce a specific form of $f_\star$ in \eqref{eq:model}:
\begin{assumption}[Single-index teacher]\label{ass:sgm}
$f_\star: \R^d \to \R$ in \eqref{eq:model} has the following form:
\begin{align}
    f_\star(\bm{x}) = \sigma_\star\left(\langle \bm{w}_\star, \bm{x} \rangle / \sqrt{d}\right),
\end{align}
with an unknown link function $\sigma_\star: \R \to \R$  and a teacher vector
$\bm{w}_\star \sim \mathcal{N}(0, \bm{I}_d)$.
\end{assumption}
The assumption of a single index as the teacher model is common in feature learning (\citet{ba2022high, moniri2024theory}). The division by $\sqrt{d}$ in the input is necessary to maintain the variance of $\langle \bm{w}_\star, \bm{x} \rangle$ at a constant order even when $d$ diverges.

Next, we consider a property for the link function $\sigma_\star$ in Assumption \ref{ass:sgm}. 
In preparation, we define the Hermite polynomial on $\R$. For $k \geq 1$, we define the $k$-th order Hermite polynomial as $H_0=1, H_1(x) = x, H_2(x) = x^2-1$, and generally
\begin{align} \label{def:Hermite}
    H_k(x) = (-1)^k  \exp(x^2/2) \frac{d^k}{dx^k} \exp(-x^2/2), ~~ x \in \R,
\end{align}
which forms an orthogonal basis in the $L^2$-space.
Then, we introduce the following condition:
\begin{assumption}[Degree of link function]\label{ass:link}
A derivative $\sigma_\star'$ of $\sigma_\star$ exists and both $ \sigma_\star $ $ \sigma_\star' $
are all polynomially bounded. 
Also, we let \( \sigma_\star \) have the following Hermite expansion in $L^2$ with $z\sim\mathcal{N}(0, 1)$:
\begin{align}
&\sigma_\star(\cdot) = \sum_{k=1}^\infty c_{\star,k} H_k(\cdot), \quad
c_{\star, k} = \frac{1}{k!}\mathbb{E}[\sigma_\star(z)H_k(z)].
\end{align}
\end{assumption}
This assumption is commonly used in feature learning for neural networks employing single indices, e.g., \citet{bietti2022learning,damian2022neural,ba2022high,cui2024asymptotics,dandi2024two}.
Given $\sigma_\star (\cdot)$, we define a simple vector of the coefficients $c_\star := (c_{\star,1},c_{\star,2},...,c_{\star,\bar{k}_\star})$ with $\bar{k}_\star := \max\{k : c_{\star,k} \neq 0\} $.

\subsubsection{Student model and training process}

We introduce conditions for the neural network $f(\cdot,; \bm{a},\bm{W})$ and their algorithms that are the subjects of training. The first concerns the activation function $\sigma: \R \to \R$:
\begin{assumption}[Degree of activation]\label{ass:activation}
Derivatives $\sigma'$ and $\sigma''$ of $\sigma$ exist, and 
\( \sigma,  \sigma',  \sigma'' \) are all polynomially bounded. 
Also, we let \( \sigma \) has the following Hermite expansion in $L^2$ with $z\sim\mathcal{N}(0, 1)$:
\begin{align}
&\sigma(\cdot) = \sum_{k=1}^\infty c_k H_k(\cdot), \quad c_{k} = \frac{1}{k!}\mathbb{E}[\sigma(z)H_k(z)].
\end{align}
We further assume that $c_{\star, 1} c_1 > 0$ holds, and there exists some $k \ge 2$, such that $c_{\star, k} c_k \ne 0$.
\end{assumption}
Given $\sigma (\cdot)$, we define a simple vector of the coefficients $c := (c_{1},c_{2},...,c_{\bar{k}})$ with $\bar{k} := \max\{k : c_{k} \neq 0\} $.

This condition is analogous to the condition introduced for $\sigma_\star$ in Assumption \ref{ass:link}. Such characterizations of link functions are also common in recent neural network theory (\citet{ba2022high, moniri2024theory}).
Regarding condition $c_{\star, 1} c_1 > 0$, analysis is similarly possible when $c_{\star, 1} c_1$ is negative. However, since symmetry yields only similar results, we avoid unnecessary redundancy in the analysis by focusing on this case.
The condition $c_{\star, k} c_k \ne 0$ for some $k \ge 2$ is formal and necessary for the analysis to properly handle the nonlinearity of the teacher and student models.

We introduce conditions for the initial values $\bm{a}^0$ and $\bm{W}^0$ for the SGD for online learning. Here, we utilize the symmetric initialization:
\begin{assumption}[Symmetric initialization] \label{ass:symmetric_init}
    The initialization $\bm{a}^0$ and $\bm{W}^0$ are set as follows:
\begin{align}
a_i^0 &= \bar{a} > 0, \quad
\bm{w}_i^0 \overset{\text{i.i.d.}}{\sim} \mathcal{N}(0,\bm{I}_d),
\quad (i=1,\dots,m).
\end{align}
\end{assumption}
This initialization is utilized by the seminal work \citet{montanari2025dynamical} for the feature unlearning, where the second layer weights are initialized as the same constant.
This scheme reduces the number of substantial order parameters and helps to obtain effective low-dimensional expressions. 
It is also possible to analyze a case with $\bar{a} < 0$; we focus on the positive $\bar{a}$ to simplify our analysis.

\section{Preparation: Reduction to fast-slow dynamics}

We derive equations described by low-dimensional variables that illustrate how neural networks update through online learning. Specifically, our goal is to derive \textit{singularly perturbed equations} capable of describing \textit{fast-slow dynamics}, which reveals a
separation of time scales in the resulting ODE system.

\subsection{ODE of macroscopic variables}

\subsubsection{Macroscopic variables of neural networks}
We introduce \textit{macroscopic variables} to obtain a tractable and low-dimensional description of neural networks by SGD. 
For $t \in \N$ and $i, j = 1,\dots, m$ with $i \ne j$, we define 
\begin{align} \label{def:macro_nn}
    R_i^m(t) &:= \plim_{n, d \to \infty}\frac{1}{d}\bm{w}_\star^\top \bm{w}_i^t,\mbox{~~and~~} 
    a_i^m(t) := \plim_{n, d\to\infty} a_i^t.
\end{align}
$R_i^m(t)$ measures the teacher-student alignment between the $i$-th weight vector and the teacher vector $\bm{w}_\star^\top$, and $a_i^m(t)$ is a scale of the $i$-th element of the second layer weight.

\subsubsection{ODE with infinite-width limit}

We next convert the difference equation to an ODE with continuous time, by considering the infinite-width limit $m \to \infty$ while the learning rate $\gamma$ is fixed. In this regime, the discrete dynamics becomes the following ODE system with continuous time $\tau = \gamma t / m \in \mathbb R_+$.

We introduce macroscopic variables $R_\tau$ and $a_\tau$ for $\tau \in \R_+$, which are continuous-time analogy of $R_i^m(t)$ and $a_i^m(t)$, respectively. 
Further, we define coefficient functions $S,T: \R \to \R$ as
\begin{align}
    S(z) = \sum_{k=1}^\infty k! c_{\star, k}c_kz^k,\quad 
    T(z) = \sum_{k=1}^\infty k! c_k^2 z^{2k}.
\end{align}
Then, we define the following ODE:
\begin{tcolorbox}
[boxrule=0mm]
    \textbf{ODE of macroscopic variables}: We define an ODE with $\{R_\tau, a_\tau\}_{\tau \in \R_+}$, as
\begin{equation}\label{eq:original}
    \begin{aligned}
        \frac{dR_\tau}{d\tau}&= \underbrace{ \frac{1}{2} a_\tau (1 - R_\tau^2) \{ 2 S'(R_\tau) - a_\tau T'(R_\tau) \}}_{ =: f(R_\tau, a_\tau) }, 
        \\ 
        \frac{da_\tau}{d\tau}&= \underbrace{ S(R_\tau) - a_\tau T(R_\tau) }_{=: g(R_\tau, a_\tau)},
    \end{aligned}
\end{equation}
    with initialization $R_0 = 0$ and $a_0 = \bar{a} > 0$.
\end{tcolorbox}
This equation allows the dynamics of neural networks with online SGD to be described by a two-variate ODE. There are already several works on representing discrete online SGD via ODE (\citet{goldt2019dynamics, collins2024hitting}).
However, while previous studies derive ODE descriptions by taking a
high-dimensional limit $d\to\infty$, we obtain the ODE by
first considering the joint limit $n,d\to\infty$ and subsequently taking the
infinite-width limit $m\to\infty$.

\subsubsection{Validation of ODE}

We prove the equivalence between the macroscopic variables of neural networks in \eqref{def:macro_nn}  and the ODE \eqref{eq:original} as follows.

\begin{proposition}\label{prop:ode_validation}
    Let $R_{\tau,i}^m := R_i^m(\lfloor m \tau / \gamma \rfloor),   a_{\tau,i}^m := a_i^m(\lfloor m \tau / \gamma\rfloor)$. Then, for any finite $\tau \ge 0$ and $i = 1,...,m$, $R_\tau, a_\tau$ satisfying the ODE \eqref{eq:original} satisfies the following asymptotic equalities
    \begin{align}
        \lim_{m\to\infty} R_{\tau,i}^m = R_\tau, \mbox{~~and~~} \lim_{m\to\infty} a_{\tau,i}^m = a_\tau.
    \end{align}
\end{proposition}

The proof is in Section \ref{sec:validation_ode}. We prove this equivalence by mediating the difference equation using the Tensor Program.

Also, in Section \ref{sec:alt_derivation}, we derive the same ODE by an alternative approach of analyzing the population gradient.

\subsection{Fast-slow reformulation}

\subsubsection{Empirical observation of multi time scale}

We numerically solve the ODE \eqref{eq:original} and study the dynamics of $( R_\tau, a_\tau)_{\tau \in \R_+}$ as shown in Figure~\ref{fig:timescale}.
Then, we observe a pronounced separation of time scales: $R_\tau$ rapidly changes over a short initial time interval, while $a_\tau$ remains nearly constant.
After this fast transient, $R_\tau,  a_\tau$ together evolve much more slowly.
This behavior is consistently observed for different conditions.
\begin{figure}[htbp]
  \begin{minipage}{0.49\linewidth}
    \centering
    \includegraphics[width=\linewidth]{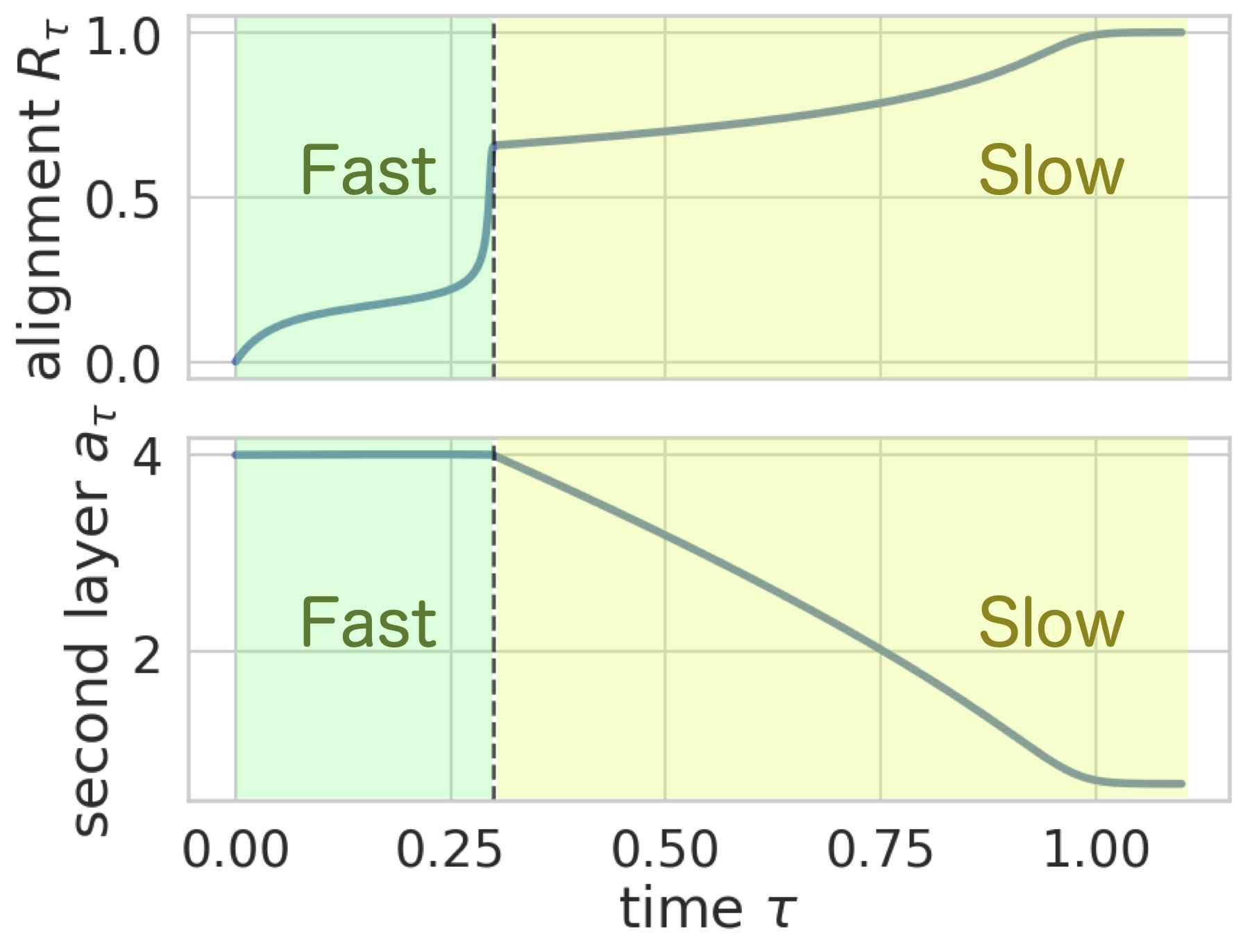}
    \label{fig:time2}
  \end{minipage}
  \hfill
  \begin{minipage}{0.49\linewidth}
    \centering
    \includegraphics[width=\linewidth]{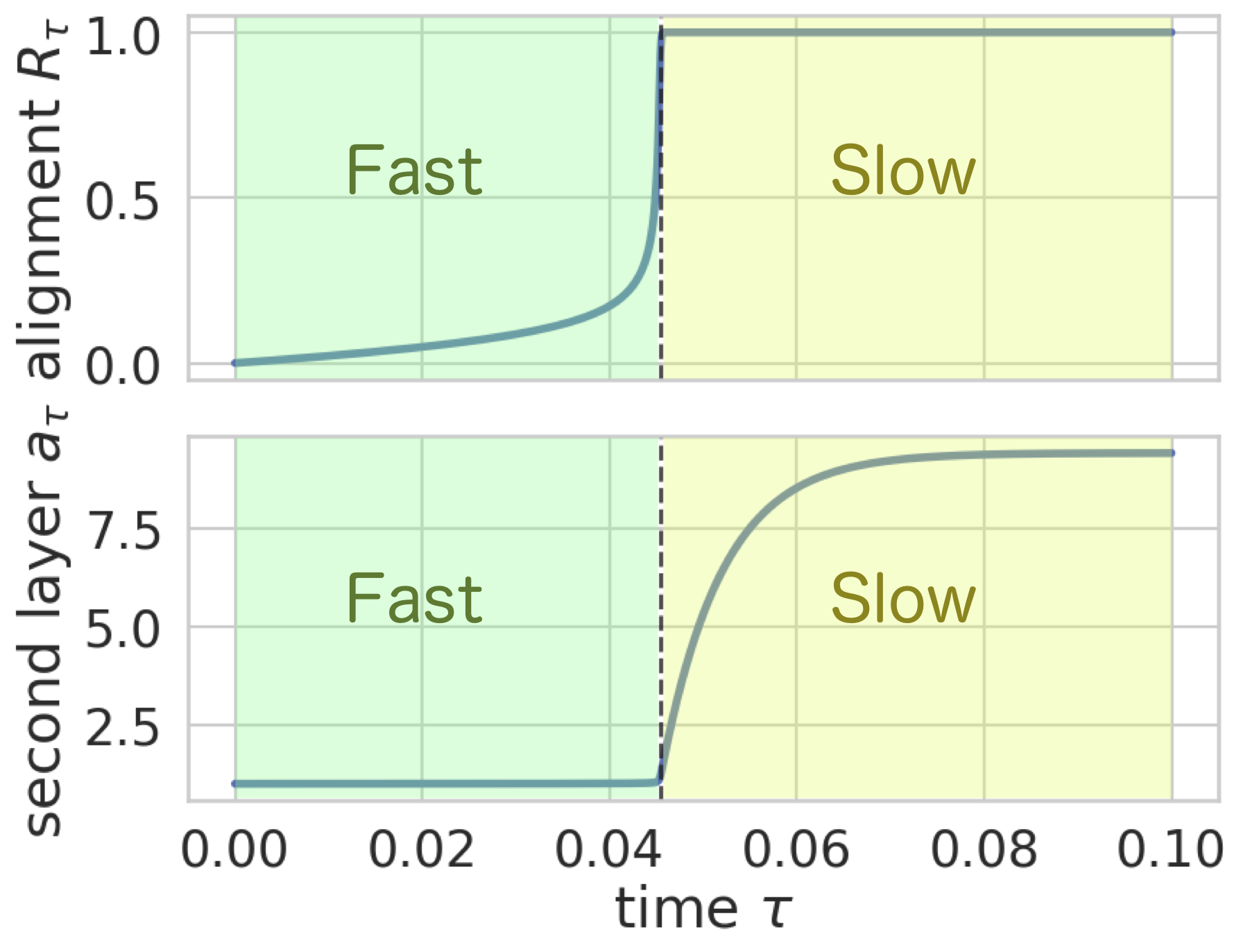}
    \label{fig:time3}
  \end{minipage}
  \caption{Multi time-scale appears in numerical simulations of \eqref{eq:original}. In the early stage of the dynamics, $R_\tau$ quickly moves away from $0$, while $a_\tau$ stays around the initial value $\bar a$. We set $\bar{k}_\star = \bar{k} = 5$ and $c = (1, 1, 1, 1, 1)$, and also set $c_{\star} = ( 1, -1, 1, -1, 1 ), \bar{a} = 4$  (left), and $c_\star = ( 2,4,6,8,10 ), \bar{a} = 1$ (right).
 \label{fig:timescale}}
\end{figure}

This difference in timescales can be explained by several insights. First, since $R_\tau$ is zero at the initial stage, the update to $a_\tau$ is nearly zero, so only $R_\tau$ is updated initially. Second, when $R_\tau$ is small, we can observe that the eigenvalues of the Jacobian of the ODE concentrate on two peaks, with the peak in the direction of updating $R_\tau$ having a larger scale. In this situation, $R_\tau$ is updated preferentially.
We will provide more quantitative discussion on this point in Section \ref{sec:origin_fast_slow}.

\subsubsection{Fast-Slow Ansatz}

Motivated by numerical observations indicating that the dynamics of
\eqref{eq:original} exhibits a pronounced separation of time scales, we
formalize this behavior by introducing two time-scales.
Rather than postulating a small parameter at the level of the vector field, we identify representative fast and slow time-scales directly from the observed
dynamics.

We introduce an ansatz for reducing \eqref{eq:original} to a singularly perturbed system. 
In preparation, we define $\lambda_f(\tau),\lambda_s(\tau) \in \R$ as eigenvalues of a Jacobian $\nabla f(R_\tau,a_\tau) $ along a solution $(R_\tau,a_\tau)$ with $\lambda_f(\tau)  \geq \lambda_s(\tau)$, corresponding to the fast and slow dynamics.
Also, $\bm v_f(\tau) \in \R^2$ denotes a normalized eigenvector corresponding to $\lambda_f(\tau)$.

\begin{ansatz}[Fast-slow dynamics]
    Given a time horizon $T$, the followings holds:
    \begin{enumerate}
  \setlength{\parskip}{0cm} 
  \setlength{\itemsep}{0cm} 
        \item (\textit{Scale separation}) Time-averaged fast/slow eigenvalues, defined as
\begin{equation}
\textstyle \Lambda_f := T^{-1}\int_0^T |\lambda_f(\tau)| d\tau \mbox{~~and~~}
\Lambda_s :=T^{-1}\int_0^T |\lambda_s(\tau)| d\tau,
\label{eq:Lambda_def}
\end{equation}
satisfy $\varepsilon := {\Lambda_s} / {\Lambda_f} \ll 1$.
        \item (\textit{Direction stability}) The eigenvector $\bm v_f(\tau)$ is well aligned to $\bm e_R = (1,0)^\top$ through time, that is, we have
\begin{equation}
\textstyle T^{-1}\int_0^T \bm v_f(\tau)^\top \bm e_R   d\tau \approx 1.
\label{eq:dir_align}
\end{equation}
    \end{enumerate}

\end{ansatz}

We can rigorously show the first point on the scale separation for small values of $|R|$: see Appendix~\ref{sec:origin_fast_slow} for details.

Such an ansatz is standard in the analysis of multi-scale dynamical
systems and has been widely used across applied mathematics and theoretical
physics.
In the context of learning dynamics, closely related approaches have been
employed \citep{jelbart2022process, montanari2025dynamical, berthier2024learning,
nishiyama2025precise}.

Based on the fast-slow ansatz above, we rewrite \eqref{eq:original} in terms of
the slow time variable $\tau_s := \Lambda_s\tau$ and $\varepsilon > 0$.
\begin{tcolorbox}[boxrule=0mm]
\textbf{Singularly perturbed system}: 
    We define $\{R_{\tau_s}^\varepsilon, a_{\tau_s}^\varepsilon\}_{\tau_s\in\mathbb
    R_+}$ by $R_{\tau_s}^\varepsilon:=R_\tau$ and
    $a_{\tau_s}^\varepsilon:=a_\tau$ as
    \begin{equation}
    \begin{aligned}
    \varepsilon \frac{dR_{\tau_s}^\varepsilon}{d\tau_s}
    &=
    f(R_{\tau_s}^\varepsilon, a_{\tau_s}^\varepsilon) / \Lambda_f
    =: \bar f(R_{\tau_s}^\varepsilon, a_{\tau_s}^\varepsilon),\\
    \frac{da_{\tau_s}^\varepsilon}{d\tau_s}
    &=
    g(R_{\tau_s}^\varepsilon, a_{\tau_s}^\varepsilon) / \Lambda_s
    =: \bar g(R_{\tau_s}^\varepsilon, a_{\tau_s}^\varepsilon),
    \end{aligned}
    \label{eq:singular}
    \end{equation}
    with initial conditions $R_0^\varepsilon=0$ and $a_0^\varepsilon=\bar a$.
\end{tcolorbox}

In this formulation, the variable $R$ evolves on the fast time scale $\Lambda_f \tau$,
while $a$ evolves on the slow time scale $\Lambda_s \tau$.
This regime differs from that considered in \citet{berthier2024learning}, where
the second-layer weights are assumed to evolve on a faster time scale than the
first-layer weights.

We view
$\varepsilon$ as an independent small parameter and consider limits 
$\varepsilon\to +0$.
Following the theory, we define the following limiting values:
\begin{align}
    a_{\tau_s}^0 :=
        \lim_{\varepsilon\to +0}a_{\tau_s}^\varepsilon, \mbox{~~and~~}R_{\tau_s}^0 :=
        \lim_{\varepsilon\to +0}R_{\tau_s}^\varepsilon.
\end{align}

\section{Feature unlearning as slow flow} \label{sec:unlearning}

We present the feature unlearning phenomenon by numerical analysis on the derived models.
While conceptual and theoretical analysis require the singularly perturbed system \eqref{eq:singular}, we utilize the ODE \eqref{eq:original} for numerical analysis because of its feasibility.

\subsection{Feature learning and critical manifold}

We first define the feature learning in the sense of the dynamics of the alignment $R_{\tau_s}^0$.

\begin{definition}[Feature unlearning] \label{def:feature_unlearning}
    We say that a neural network system follows the \textit{feature unlearning}, if the variable for the alignment $\{R_{\tau_s}^0\}_{\tau_s \in \R_+}$ satisfies the following:
    there exists a constant $\bar{c} > 0$ and finite $\bar{\tau}$ such that we have
    \begin{align}
        \max_{\tau_s \in (0,\bar{\tau})} |R_{\tau_s}^0 | = \bar{c}, \mbox{~~and~~} \lim_{\tau_s \to \infty} |R_{\tau_s}^0| =0.
    \end{align}
\end{definition}
In contrast, when $\lim_{\tau_s \to \infty} |R_{\tau_s}^0|$ is lower bounded by a strictly positive constant, we say that the neural network achieves \textit{feature learning}.
This definition of feature unlearning implies that the first-layer of a neural network aligns to the teacher vector $\bm{w}_\star$ as a feature at an early stage, and then the learned feature may be lost as training progresses. This definition conceptually follows  \citet{montanari2025dynamical}. 

Further, we formally define a manifold in the $R-a$  space,  here termed a \textit{critical manifold}, to which $\{R_{\tau_s}^0, a_{\tau_s}^0\}_{\tau_s \in \R_+}$ stays close in slow time.
\begin{definition}[Critical manifold] \label{def:critical_manifold}
    We define a critical manifold $\mathcal{S}$ by $\bar{f}(\cdot,\cdot)$ in the singularly perturbed system \eqref{eq:singular} as
\begin{align}
    \mathcal{S}:=\{(R,a) \in [-1,1] \times \R \mid \bar{f}(R,a)=0\}.
\end{align}
\end{definition}
As illustrated in Figure \ref{fig:feature_outline} and subsequent numerical analysis, $\mathcal{S}$ becomes a continuously smooth one-dimensional manifold.
Note that $R_{\tau_s}^0$ takes a value only within $[-1,1]$ by the form of ODE and the normalization of the online SGD \eqref{def:sgd_w}, hence it is sufficient to consider $R \in [-1,1]$.

\subsection{Slow flow on the critical manifold} \label{sec:feature_emp}

We analyze the the dynamics of the macroscopic variables $(R_\tau, s_\tau)$ in the space or $R$ and $a$, i.e. an $R-a$  space.
Numerically, we utilize the ODE \eqref{eq:original} as a proxy of the singularly perturbed system \eqref{eq:singular}, which is a common approach for the singular perturbation theory \citep{hek2010geometric} and its application to machine learning \citep{serino2025fast,patsatzis2024physics}.

The result, illustrated in Figure \ref{fig:geometry}, reveals that there are two types of dynamics.
In the fast time scale, the alignment $R_{\tau}$ evolves rapidly while $a_{\tau}$ remains effectively frozen, and trajectories are attracted to the critical manifold $\mathcal{S}$.
Then, on longer time scales, the evolution is governed by a reduced slow flow along $\mathcal{S}$.
Direction of the slow flow on $\mathcal{S}$ is determined by the link $\sigma(\cdot)$ and activation$\sigma_\star(\cdot)$, and if there are unstable points on $\mathcal{S}$, the direction changes there.
Even in dynamics on $\mathcal{S}$, some trajectories exhibit that they left $\mathcal{S}$ and only $R_\tau$ evolves independently of $\mathcal{S}$. In such cases, the slow dynamics resumes upon reaching $\mathcal{S}$ again.

By further analysis on the longer time scale, we can observe two types of trajectories: \textbf{(I) one converges to a finite point} $\lim_{\tau \to \infty} (R_\tau,a_\tau) \to (R',a') \in \mathcal{S}$, and \textbf{(II) one diverges with zero alignment}, i.e. it behaves as $\lim_{\tau \to \infty} (R_\tau,a_\tau) \to (0,\pm\infty)$. 
Which trajectory appears depends on which point on the critical manifold $\mathcal{S}$ is reached on the fast time scale.

Consequently, the trajectory (II), which has diverging dynamics on $\mathcal S$, exhibits feature unlearning.
Specifically, along certain branches the reduced dynamics makes $|R_\tau| > 0$ once and afterwards drives $R_\tau$ gradually toward zero. 
This behavior corresponds to a progressive loss of alignment $R_{\tau_s}^0$, and can be naturally interpreted as
feature unlearning emerging from the slow drift along the attracting branch. 
We give more plots with different choice of $\sigma(\cdot)$ and $\sigma_\star(\cdot)$ in Section \ref{sec:additional_simulation}.

\begin{figure}[htbp]
  \centering
  \begin{minipage}{0.49\linewidth}
    \centering
    \includegraphics[width=\linewidth]{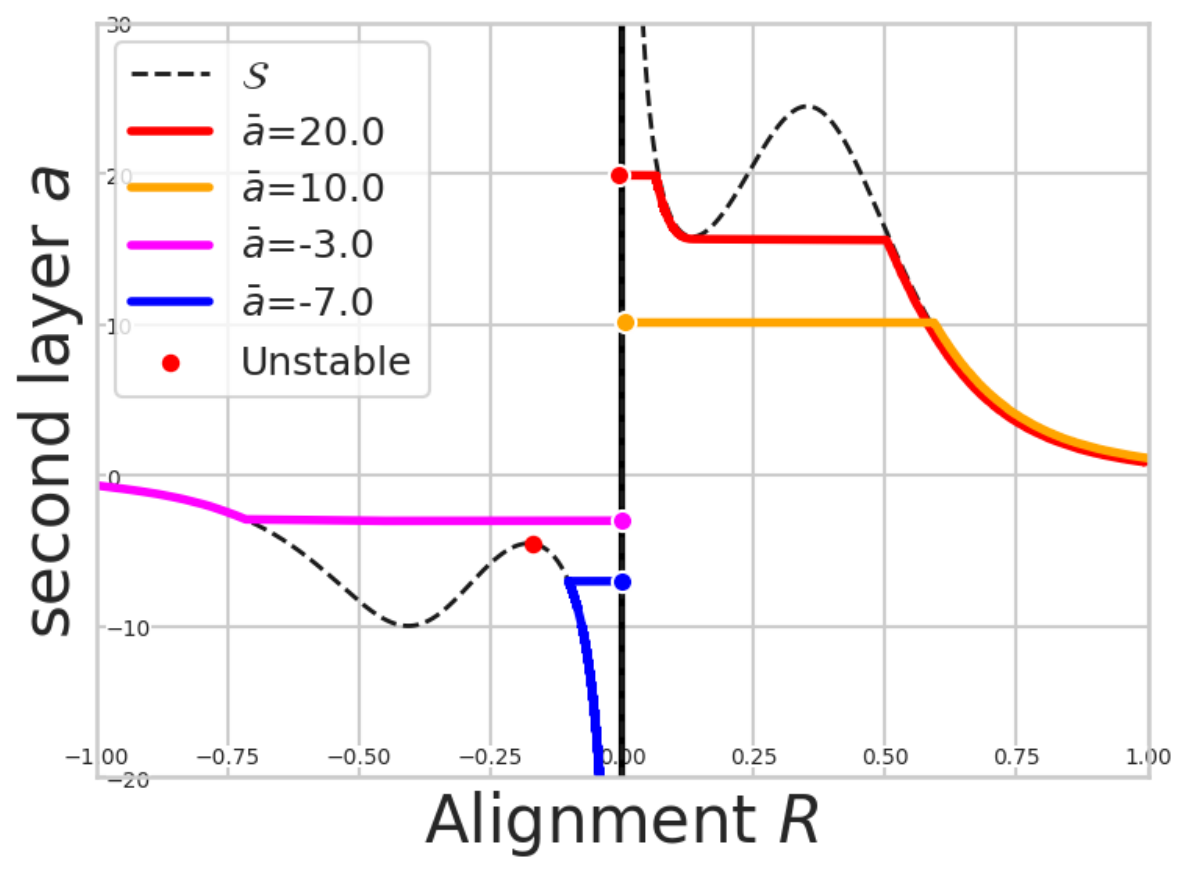}
    \label{fig:geometry1}
  \end{minipage} 
  \hfill
  \begin{minipage}{0.49\linewidth}
    \centering
    \includegraphics[width=\linewidth]{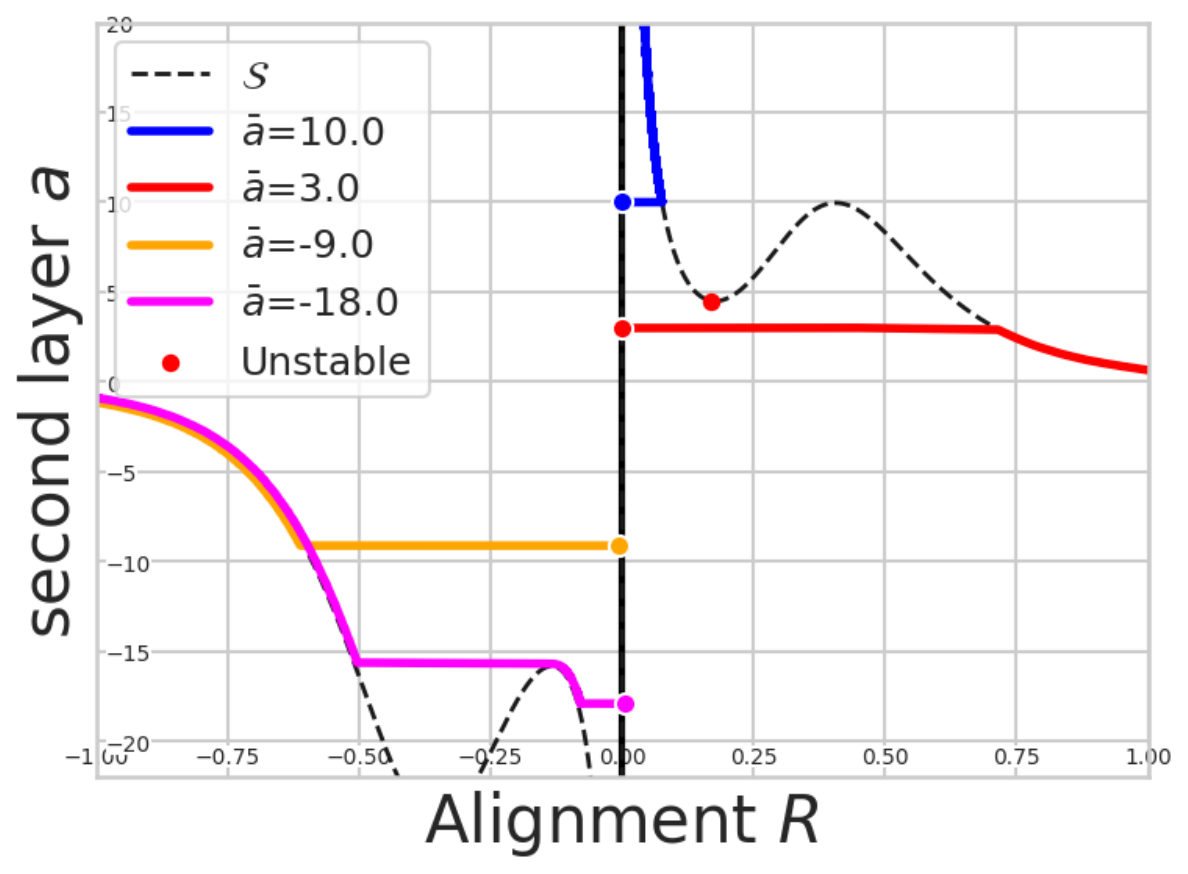}
    \label{fig:geometry2}
  \end{minipage}
  \caption{Trajectories of the model \eqref{eq:original} on the $R-a$ space. The dots on the $a$-axis are the initial values $\bar a$. The red, yellow, pink trajectories show feature learning, and the blue trajectories shows feature unlearning.
  We set $\bar{k}_\star = \bar{k} = 5$ and $c = (1, 1, 1, 1, 1)$, and also set $c_{\star} = ( 1, 1, 1, 1, 1 )$ (left) or $c_\star = ( 1, -1, 1, -1, 1 )$ (right).
\label{fig:geometry} }
\end{figure}

\subsection{Fast-slow flow and loss dynamics}

We study the connection between the fast-slow dynamics in the $R-a$ space described above and the dynamics of other variables, such as the test loss. 
Figure \ref{fig:unlearning_manifold_numerical} compares the trajectory on a trajectory on the $R-a$ space with the corresponding transitions of the alignment $R_\tau$, the second-layer scaler $a_\tau$, and the test loss.

In the initial fast dynamics where alignment $R_\tau$ increases from $0$, a rapid improvement in $R_\tau$ and a rapid decrease in loss occur. Subsequently, in the dynamics (I) where feature learning occurs, as $R_\tau$ increases slowly, the loss also gradually decreases. Furthermore, when the trajectory temporarily leaves the critical manifold $\mathcal{S}$ and evolves rapidly, the loss also decreases rapidly. In contrast, in the dynamics (II) for feature unlearning occurs which diverges toward $R_\tau$ is zero, causing $R_\tau$ to decrease monotonically. The loss continues to decrease but converges to a value of the lazy regime value without performing feature learning.

\begin{figure*}[htbp]
    \centering
    \includegraphics[width=0.44\linewidth]{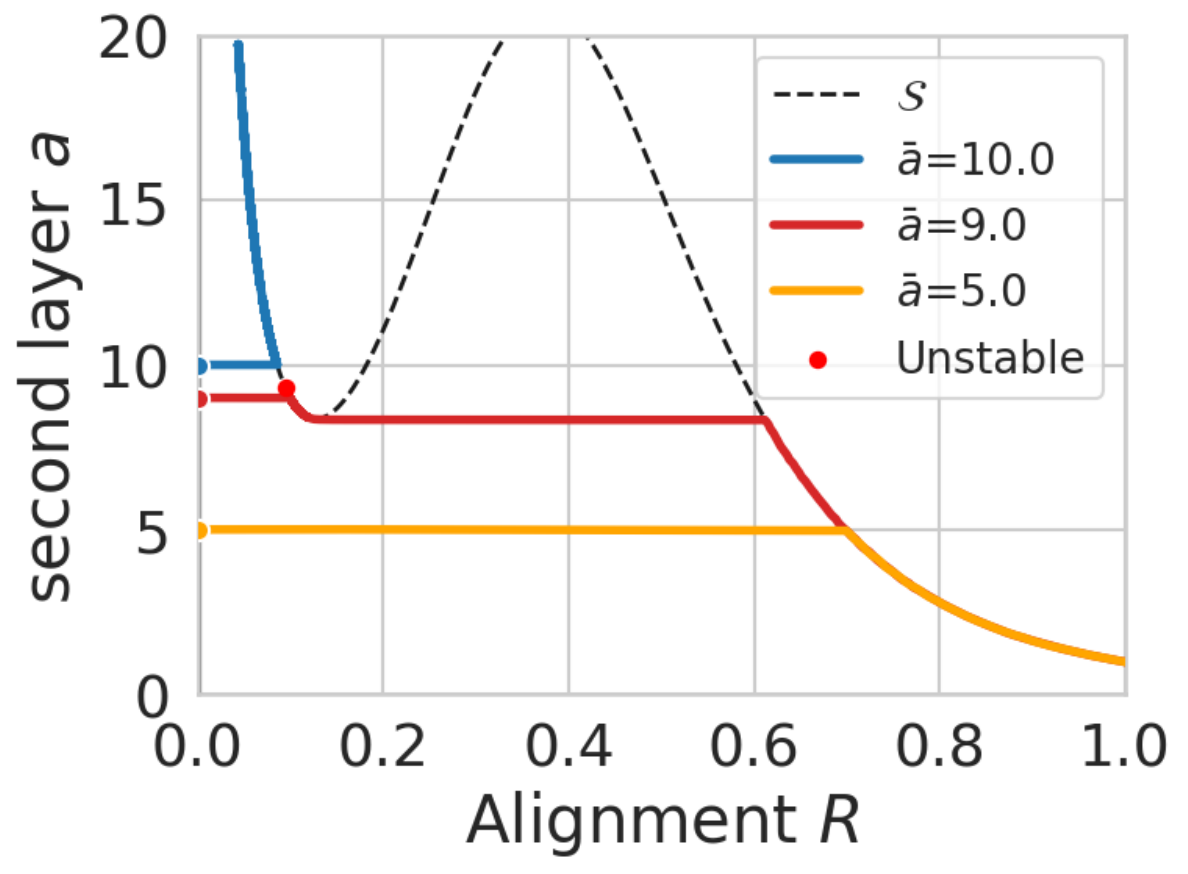}
    \includegraphics[width=0.44\linewidth]{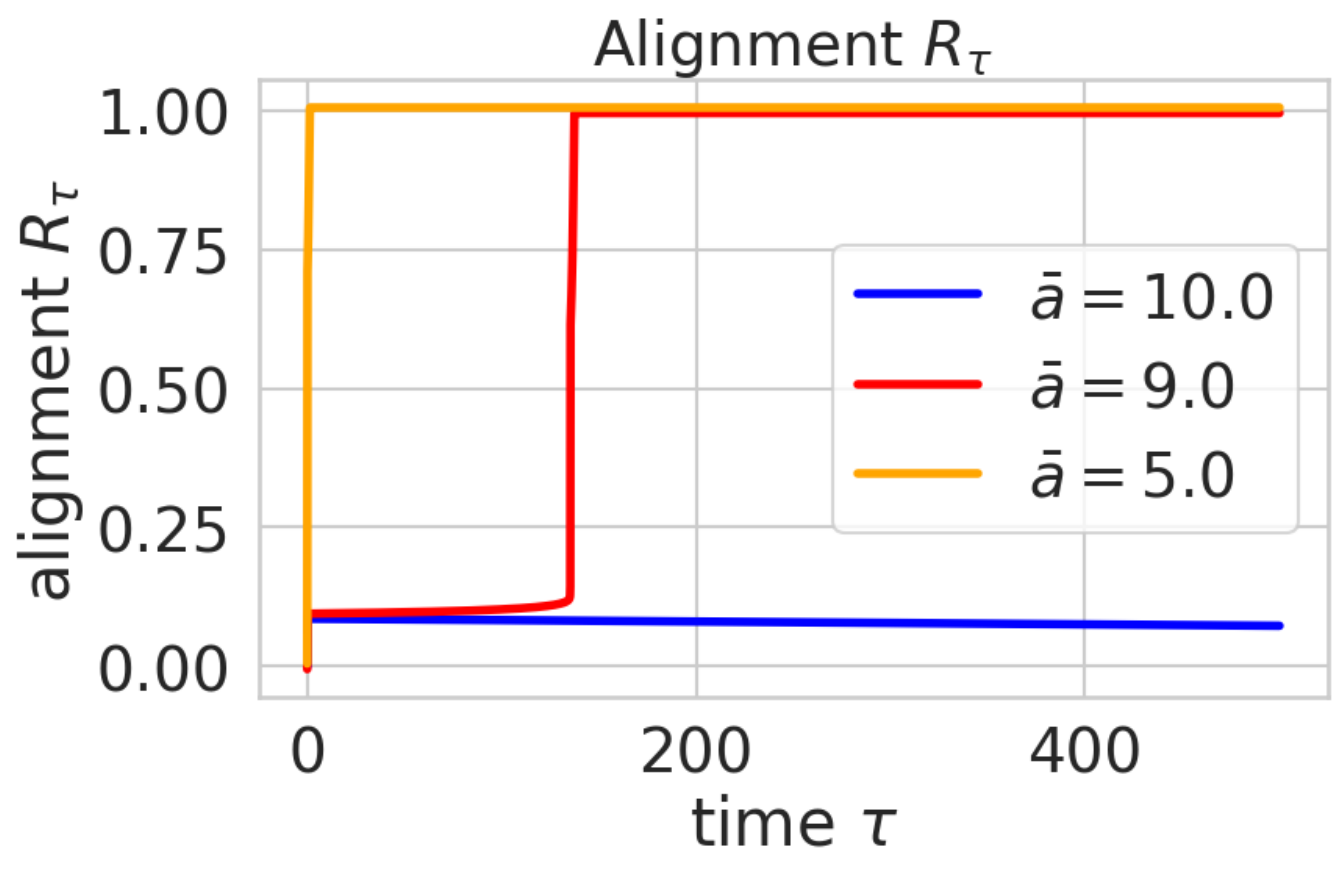} \\
    \includegraphics[width=0.44\linewidth]{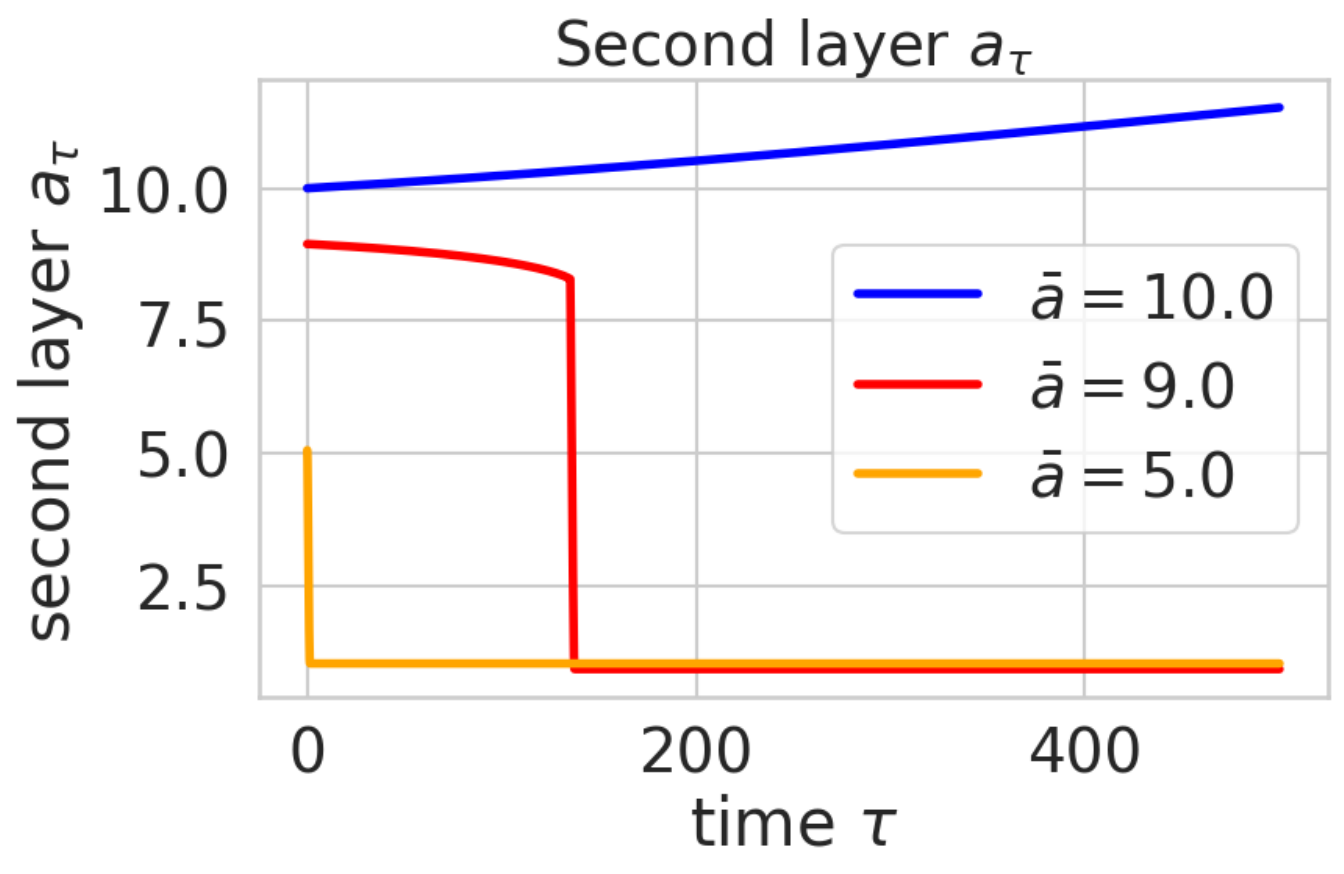}
    \includegraphics[width=0.44\linewidth]{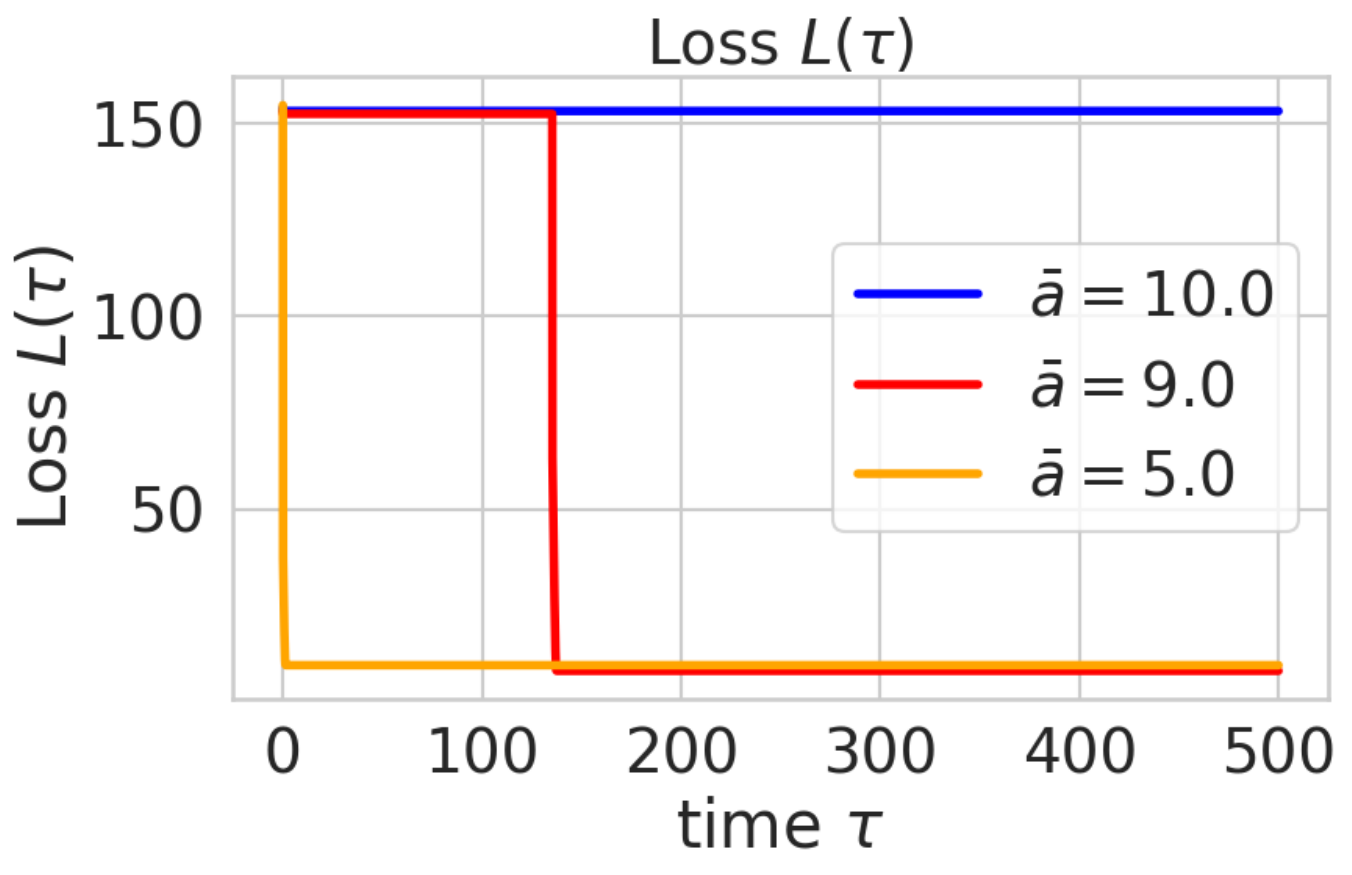}
    
    \caption{Simulated trajectories, alignments, second-layer weights, and losses of the model \eqref{eq:original}. We set $\bar{k}_\star = \bar{k} = 5$ and $c = (1, 1, 1, 1, 1), c_{\star} = ( 1, -1, 1, 1, 1 )$, and  $\bar{a} \in \{5, 9, 10\}$. We can observe that the learning dynamics proceeds differently for each case.}
\label{fig:unlearning_manifold_numerical}
\end{figure*}

\section{Theoretical grounding} \label{sec:theory}

We derive a theoretical description of the feature learning based on the singularly perturbed system \eqref{eq:singular}.
This analysis mathematically supports the observation on the feature learning in Section \ref{sec:feature_emp}.

\subsection{Condition of feature unlearning}

We provide a rigorous justification of feature unlearning in Definition \ref{def:feature_unlearning}. 
First, we put an assumption on the relation between the link $\sigma_\star(\cdot)$ and the activation function $\sigma(\cdot)$:
\begin{assumption}[Redundant degree of polynomial of activation]\label{ass:alpha}
With $k_0:=\min\{k\ge2:\ c_{\star,k}c_k\neq0\}$ and $ 
k_1 = \min \{ k \ge 2; c_k \ne 0 \}$, one of the followings holds:
\begin{itemize}
  \setlength{\parskip}{0cm} 
  \setlength{\itemsep}{0cm} 
    \item[(i)] $k_0 + 1 > 2k_1$,
    \item[(ii)] $k_0 + 1 < 2k_1$ and $c_{\star, k_0} c_{k_0} < 0$,
\end{itemize}
\end{assumption}
Both of these conditions refer to a situation where the student model possesses low-order nonlinearities that are not present in the teacher model.

We next introduce an assumption on the initialization $\bar{a}$ through functions which may describe the dynamics on the critical manifold $\mathcal{S}$.
Here, to simplify the analysis, we consider the case with $R > 0$.

\begin{assumption}[Initialization on $\bar{a}$]\label{ass:bar_a}
We define $h,\alpha:(0,1)\to \R$ as
\begin{align}
    h(R)
= \frac{2S'(R)}{T'(R)} \mbox{~~and~~}\alpha(R):=S(R)T'(R)-2S'(R)T(R).
\end{align}
Also, we define the following values $R_h = \min \{R \in (0,1)\mid h'(R) = 0 \}$ and $R_\alpha = \min \{R \in (0,1)\mid \alpha(R) = 0 \},$ with considering $\min \emptyset = 1$, and also define $R^\star = \min\{R_h , R_\alpha\} $.
Then, we assume that, there exists some $R \in (0, R^\star)$ such that $\bar{a} = h(R)$ holds.
\end{assumption}
This assumption requires that the initial value $\bar{a}$ lies in a region that induces unlearning, i.e., divergence of the dynamics on the $\mathcal{S}$. Here, $h(\cdot)$ is the parameterization of $a_{\tau_s}^0$ on $\mathcal{S}$ with respect to $R_{\tau_s}^0$, and $\alpha (\cdot)$ is a component of the intrinsic dynamics of $R_{\tau_s}^0$ on $\mathcal{S}$. These characterize the region of $(R,a) \in \mathcal{S}$ that moves toward divergence.

More precisely, direction of the slow flow on $\mathcal{S}$ is determined by Assumption \ref{ass:bar_a}.
We can observe that the direction of the slow dynamics on $\mathcal S$ is entirely determined by the sign of $\alpha(h^{-1}(a))$. Consequently, when $h^{-1}(\bar a)$ crosses a root of $\alpha(R)=0$ due to a change in the initial value $\bar a$, the direction of the slow flow is reversed. In particular, since Assumption~\ref{ass:alpha} guarantees that $\alpha(R) > 0$ holds in a neighborhood of $R = +0$, feature unlearning occurs when $h^{-1}(\bar a)$ is smaller than the smallest positive root of $\alpha(R) = 0$.
We can find details from the reduced ODE (\eqref{eq:slow_detail}) in Section \ref{sec:proof_main_theorem}.

With these assumptions, we now state the theorem for feature unlearning:

\begin{theorem}[Feature unlearning]\label{thm:main}
    Under Assumptions \ref{ass:link}-\ref{ass:bar_a},  we obtain
    \begin{align}       \lim_{\tau_s\to\infty}R_{\tau_s}^0 = 0,~    \mbox{and} ~    \lim_{\tau_s\to\infty}a_{\tau_s}^0 = \infty.
    \end{align}
    Furthermore, we have $\lim_{\tau_s\to\infty}R_{\tau_s}^0 a_{\tau_s}^0 = {c_{\star, 1}} / {c_1}$.
\end{theorem}
This theorem derives a sufficient condition for feature unlearning to occur. Specifically, under the assumptions imposed here, the divergence of $a_{\tau_s}^0$ and the vanishing of $R_{\tau_s}^0$, which represents alignment, indicate that the first layer weights lose the learned features, meaning learning occurs in the so-called lazy regime. The third limit of $R_{\tau_s}^0  a_{\tau_s}^0 $ provides further additional information, that is, the rate of divergence of $a_{\tau_s}^0$ is of the order $O((R_{\tau_s}^0)^{-1})$.

This result can be regarded as a more precise description of the conditions under which feature unlearning occurs, as demonstrated in the \citet{montanari2025dynamical} under the similar setting. 
Furthermore, using a similar proof, it is also possible to derive a sufficient condition under which feature unlearning does not occur. 

\subsection{Scaling law of feature unlearning}

We derive a scaling law of the variables $R_{\tau_s}^0$ and $a_{\tau_s}^0$ in the feature unlearning case, which shows their convergence rate in terms of $\tau_s \to \infty$.
\begin{theorem}[Scaling law]\label{thm:scaling}
    Under Assumptions \ref{ass:link}-\ref{ass:bar_a},
    for each case of (i) and (ii) in Assumption \ref{ass:alpha}, we obtain the following as $\tau_s \to \infty$:
    \begin{itemize}
  \setlength{\parskip}{0cm} 
  \setlength{\itemsep}{0cm} 
        \item[(i)] $R_{\tau_s}^0 = \Theta(\tau_s^{-{1}/{(2k_1)}}), \mbox{~~~and~~~}  a_{\tau_s}^0 = \Theta(\tau_s^{{1}/{(2k_1)}})$,

        \item[(ii)] $ R_{\tau_s}^0 = \Theta(\tau_s^{-{1}/({k_0 + 1})}), \mbox{~~~and~~~} a_{\tau_s}^0 = \Theta(\tau_s^{{1}/({k_0 + 1})})$,
    \end{itemize}
\end{theorem}
These results imply that 
$k_0$ and $k_1$ defined in Assumption \ref{ass:alpha} are essential in determining the speed of convergence.

\section{Simulation}

\subsection{Phase map of feature unlearning}

We perform numerical simulations of the ODE \eqref{eq:original} for multiple
choices of activation and link function coefficients $(c_k,c_{\star,k})$.
Based on these simulations, we also construct a phase diagram summarizing whether unlearning occurs as a function of the coefficients and the initial value $\bar{a}$.
Figure \ref{fig:phase_map} shows the result. We see that, in this case, sign matching between the teacher/student coefficient is important for successful feature learning.

\begin{figure}[htbp]
    \centering
    \includegraphics[width=0.4\linewidth]{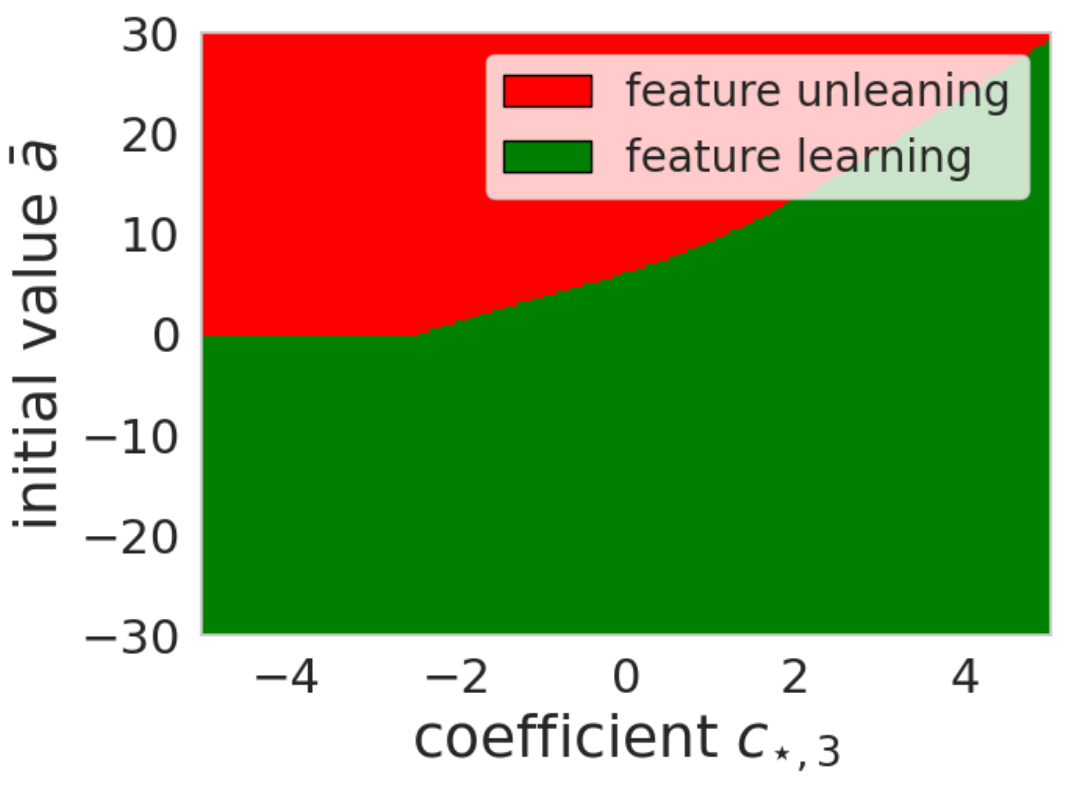}
    \includegraphics[width=0.4\linewidth]{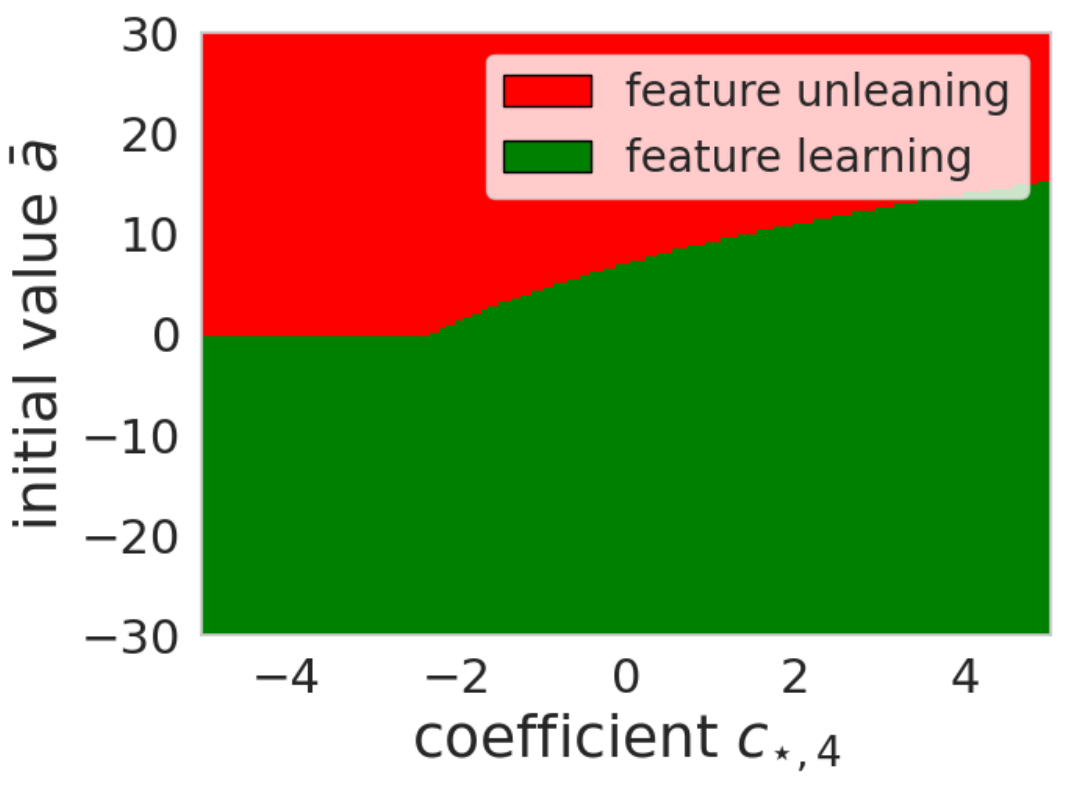}
    \caption{Phase maps for the feature unlearning by  \eqref{eq:original}. We set $\bar{k}_\star = \bar{k} = 5$ and $c = (1, 1, 1, 1, 1)$, and also set $c_{\star, 1} = c_{\star, 4} = c_{\star, 5} = 1, c_{\star, 2} = -1$  (left),  or $c_{\star, 1} = c_{\star, 2} =  c_{\star, 5} = 1, c_{\star, 3} = -1$ (right).}
    \label{fig:phase_map}
\end{figure}

\subsection{Scaling law of feature unlearning}

We numerically investigate the convergence rates predicted by
Theorem~\ref{thm:scaling}.
By tracking the long-time behavior of $R_{\tau}$ and $a_{\tau}$ in the
ODE \eqref{eq:original} , we find clear power-law regimes whose exponents agree with the
theoretical scalings.
These results provide quantitative confirmation of the scaling law for feature
unlearning derived from the singular perturbation analysis. We observe that, from Figure \ref{fig:scaling}, the log-log tail slopes of $R_\tau$ and $a_\tau$ of both settings get close to the theoretical values $\pm1/4$, $\pm1/3$ respectively.

\subsection{Experiments with real neural networks and SGD}
We report simulations of online SGD applied directly to real
two-layer neural networks.
Figure \ref{fig:agreement} shows the results. 
Across all tested configurations, we observe qualitative behaviors consistent with fast-slow dynamics, including a gradual decay of
alignment with the teacher direction, accompanied by growth in the second-layer weights.
While finite-width effects and stochastic fluctuations remain visible, these
results suggest that the fast-slow mechanism predicted by the infinite-width theory persists, at least transiently, in realistic neural network settings.

\begin{figure}[htbp]
    \centering
    \includegraphics[width=0.4\linewidth]{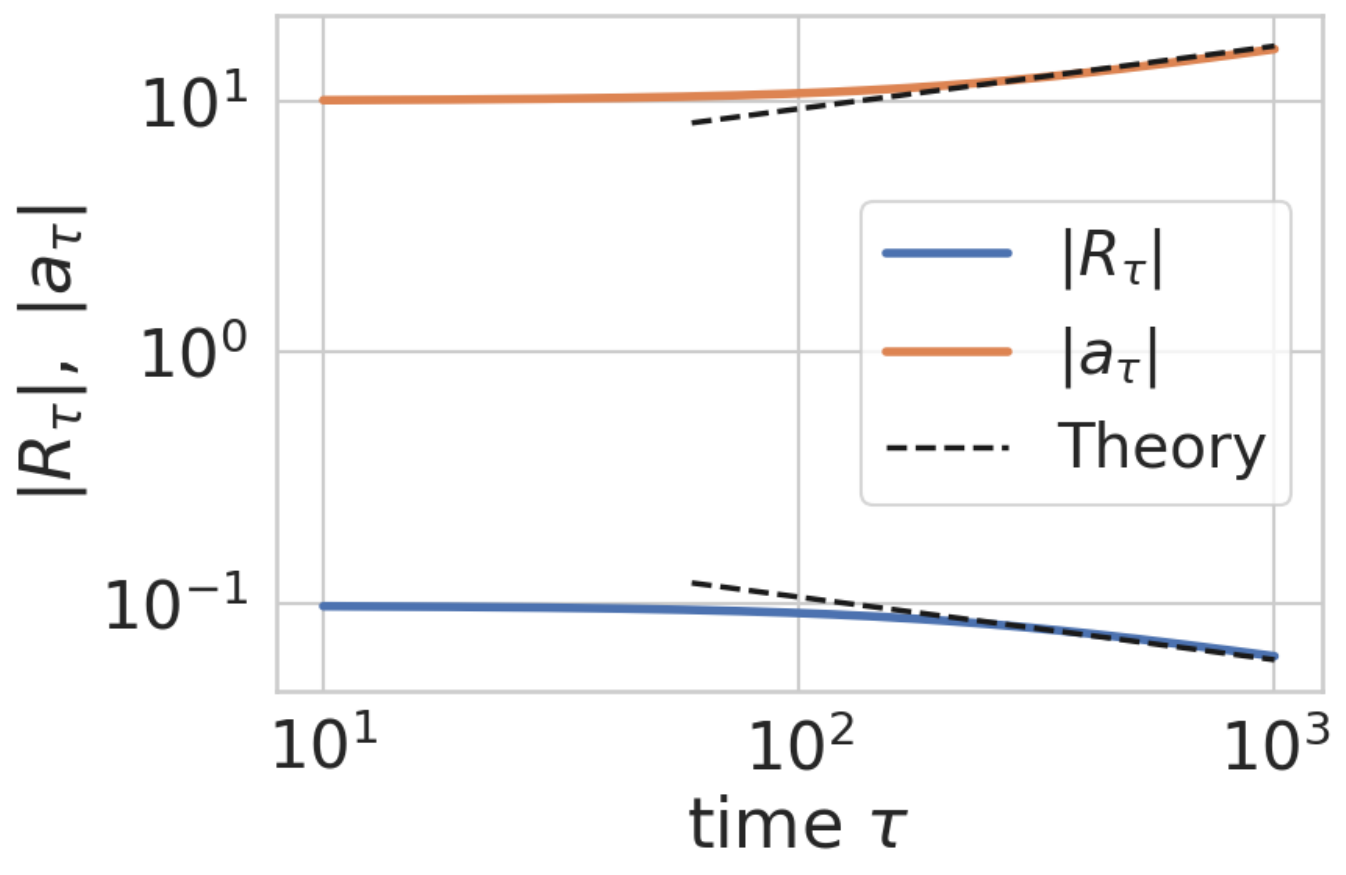}
    \includegraphics[width=0.4\linewidth]{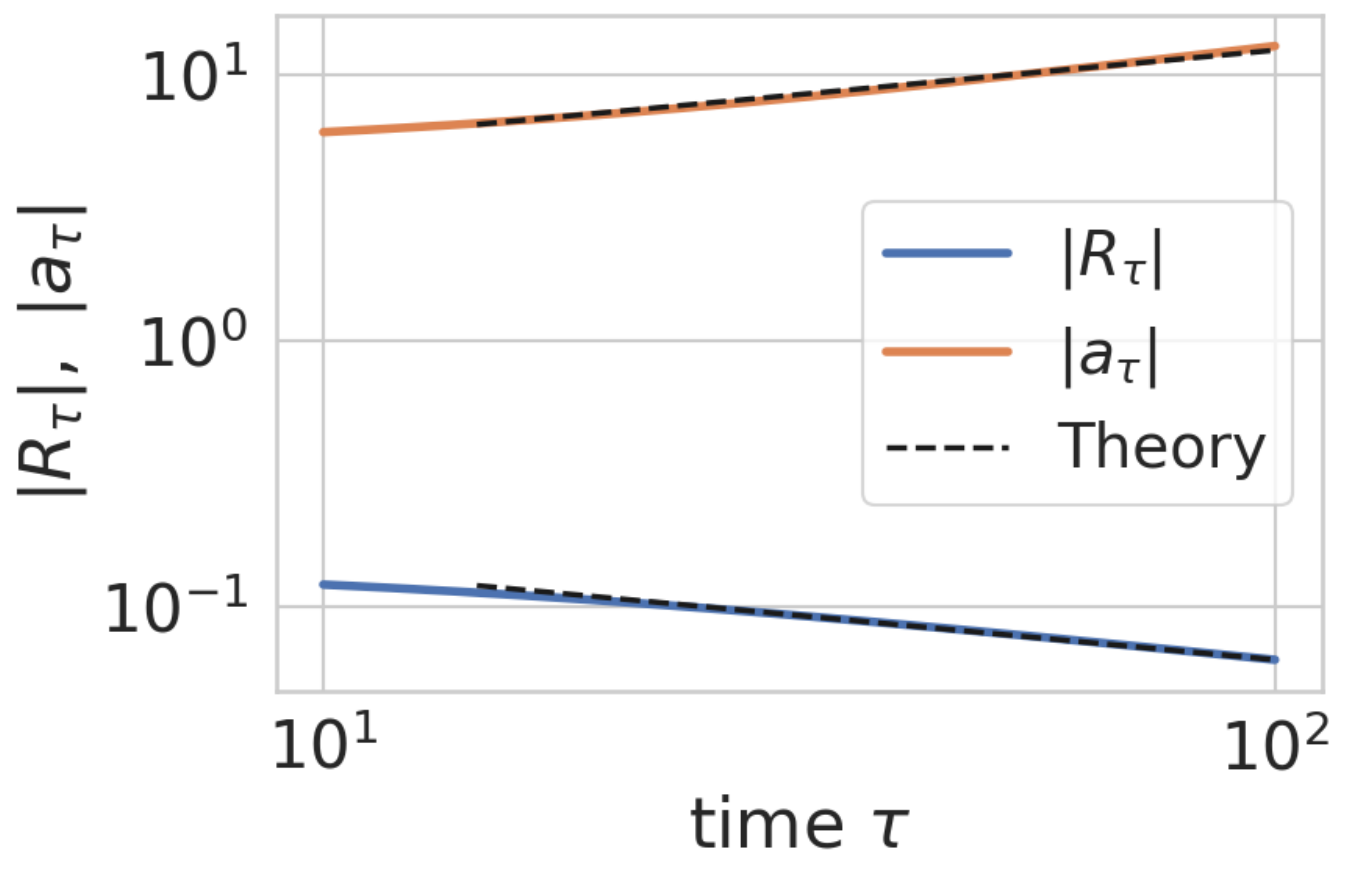}
    \caption{Numerical verification of the scaling law of Theorem \ref{thm:scaling}.  ({Left}) $\bar{k}_\star = \bar{k} = 7$, $c = (1, 1, 1, 1, 1, 1, 1), c_\star =( 1, 0, 0, 0, 0, 0, 0.5),\bar a = 10$. This corresponds to the case (i) of Assumption \ref{ass:alpha}; ({Right}) $\bar{k}_\star = \bar{k} = 3$, $c = (1, 1, 1), c_\star = ( 1, -1, 1 ), \bar a = 5$. This corresponds to the case (ii) of Assumption \ref{ass:alpha}.}
    \label{fig:scaling}
\end{figure}

\begin{figure}[htbp]
  \centering
  \begin{minipage}{0.49\linewidth}
    \centering
    \includegraphics[width=\linewidth]{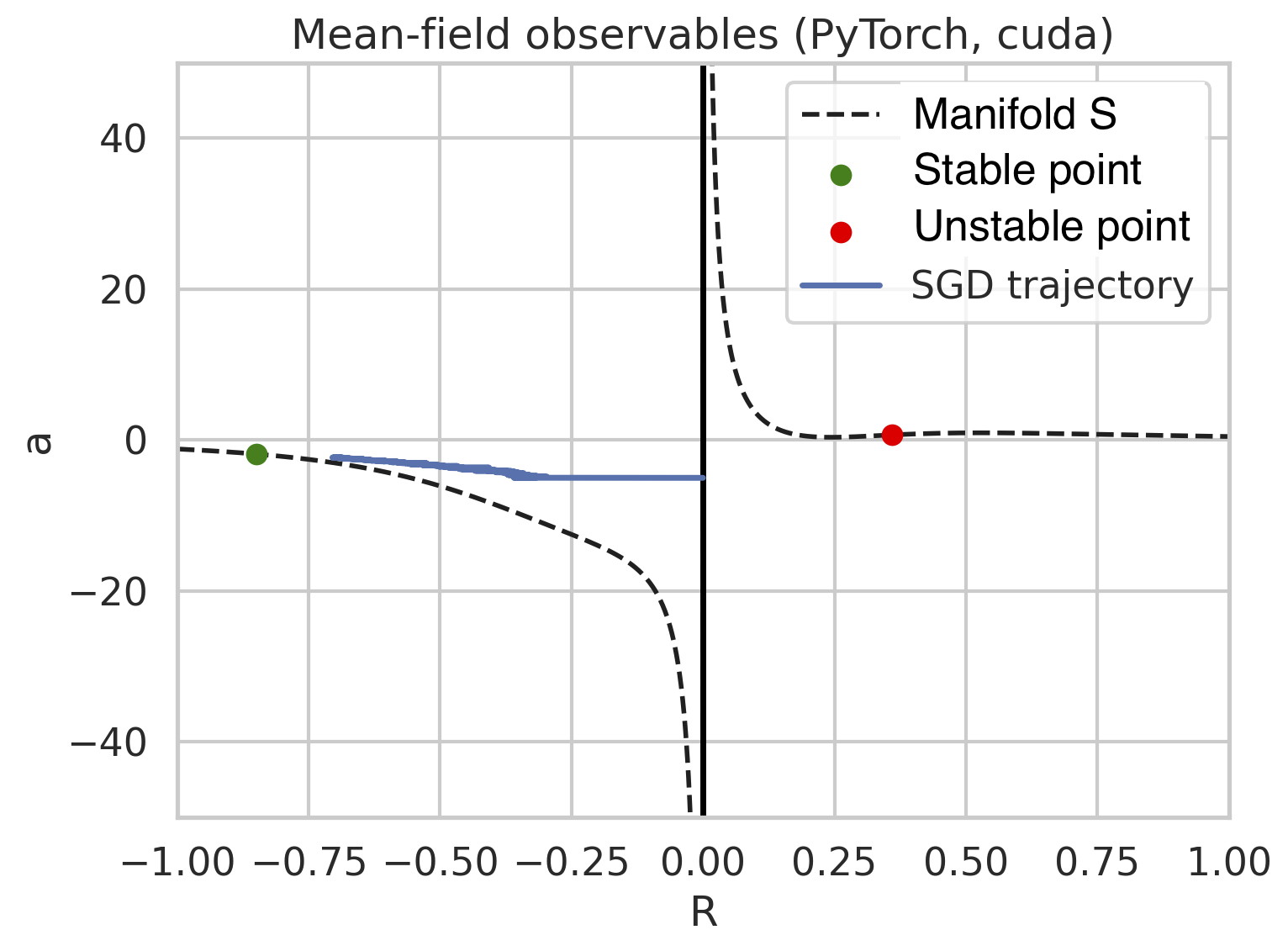}
  \end{minipage}
  \hfill
  \begin{minipage}{0.49\linewidth}
    \centering
    \includegraphics[width=\linewidth]{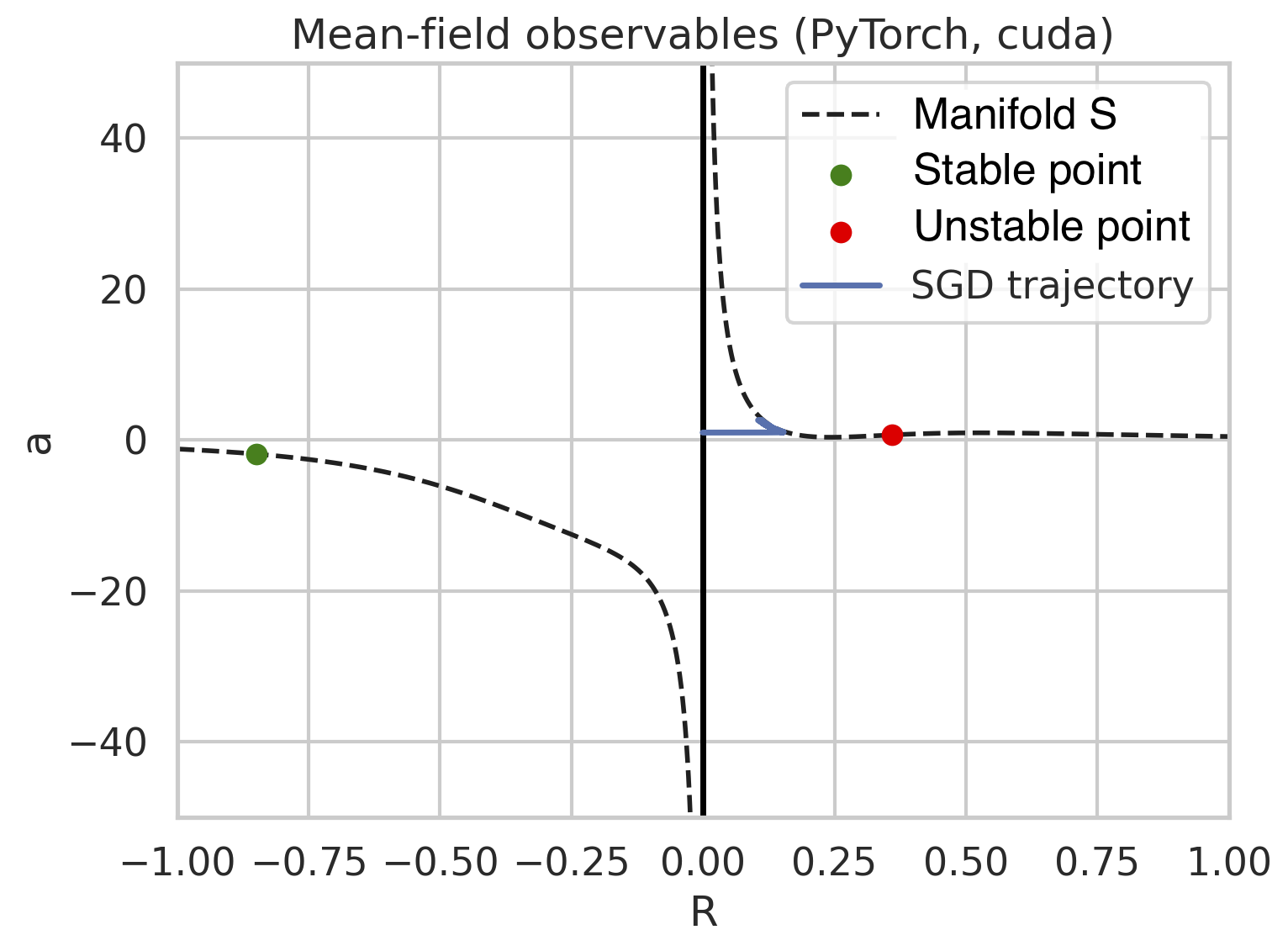}
  \end{minipage}
  \caption{Numerical simulations of real neural networks up to $10^4$ iterations for $\bar{k}_\star = \bar{k} = 3$, $c = (1, 1, 1), c_\star = (1, -2, 1)$, and learning rate $\gamma = 1$ with ({Left}) $n = d = 10^4, m= 500$;   ({Right}) $n = d = 10^4, m = 10^3$. Stable/unstable points are where the flow direction on the manifold changes.\label{fig:agreement}}
\end{figure}

\section{Conclusion}

This paper investigates the learning dynamics via SGD using data obtained from a single-index teacher model to elucidate the mechanism behind feature unlearning in two-layer neural networks. We represented the dynamics of macrovariables of the neural network as continuous-time differential equations, then analyze the difference in timescales of these variables. As a result, we clarified the relationship between the fast-slow dynamics of the macrovariables and the critical manifold, demonstrating that the slow flow on the manifold explains phenomena such as feature unlearning and staircase loss dynamics.

\appendix

\section{Related works}

\subsection{Feature learning/unlearning in two-layer neural networks}
The dynamics of feature learning in two-layer neural networks has been studied
extensively in teacher-student settings and high-dimensional regimes (\citet{ba2022high, damian2022neural, moniri2024theory, yang2021tensor4}).
The phenomenon of feature unlearning, where alignment with previously learned
features degrades over long training times, was recently identified in large
two-layer neural networks by \citet{montanari2025dynamical}.
Using the dynamical mean-field theory, these authors revealed a pronounced separation of time scales and
showed that feature unlearning can occur even under full-batch gradient flow in
the infinite-width limit.

\subsection{Tensor Programs and other theories for high-dimensional dynamics}
Tensor Programs provide a constructive and algorithm-aware framework for
deriving macroscopic descriptions of wide neural networks directly from their
computational graphs and update rules.
Originally developed to analyze forward-pass behavior and signal propagation
in deep networks \citet{yang2019scaling,yang2019tensor1}, the framework was later
extended to cover backpropagation, training dynamics, and a more general program
structures \citet{yang2020tensor2,yang2020tensor3,yang2021tensor2b,yang2021tensor4,
yang2022tensor5,littwin2023adaptive,yang2024tensor6}.
Recent work has conducted the analysis of discrete gradient-based training algorithms,
yielding rigorous state-evolution results for stochastic gradient descent and
related methods \citet{gerbelot2024rigorous,dandi2024two}.
In contrast to the dynamical mean-field theory, which typically postulates a closed-form dynamical
description at the outset, Tensor Programs offers a systematic procedure for deriving such
descriptions from the underlying algorithm, a perspective that is central to
the present work.

Recent works have focused on the high-dimensional limit, where learning
dynamics can be characterized through a small number of macroscopic order
parameters.
Within this framework, both gradient flow and stochastic gradient descent have
been studied using tools from dynamical mean-field theory, revealing the
dependence of learning behavior on initialization, width, and data statistics (\citet{bordelon2022self, celentano2021high, dandi2024two, mignacco2020dynamical}).
Another theory is the generalized first-order theory, which handles a general class of iterative algorithms with first-order gradients and derives a state evolution to represent a limiting high-dimensional dynamics of the iterative algorithms. \citet{celentano2020estimation} developed the framework and its efficiency, \citet{han2025entrywise} studies its applicability to a wider class of models and data distributions, and \citet{han2025precise} applied the theory to the dynamics of multi-layer neural networks.

\subsection{Singular perturbation theory}
The singular perturbation theory is a general theory to analyze
multi-scale phenomena in dynamical systems and have long been used in physics, chemistry, biology, and many others \citep{fenichel1979geometric, patsatzis2024slow, serino2025fast, jelbart2022process, kojakhmetov2021geometric, tran2025singular}.
Regarding the analysis for neural networks, \citet{berthier2024learning} analyzed incremental and non-monotone learning dynamics by explicitly introducing a small parameter to utilize the singular perturbation theory.
\citet{montanari2025dynamical} applied the singular perturbation theory for two-layer neural networks with gradient flow and analyzed the feature unlearning phenomenon as described above. 
\citet{nishiyama2025precise} studied a diagonal linear neural network using singular perturbation theory and developed a precise dynamical analysis of the network training.

\section{Proof of Validation of ODE} \label{sec:validation_ode}

\subsection{Difference equation of macroscopic variables for finite-width network}

We derive the difference equation representing the dynamics of the macro variables. This equation is derived using the framework of Tensor Programs \citep{yang2019scaling,yang2020tensor2,yang2021tensor2b}, based on the neural network \eqref{eq:neural_network} and the online SGD update \eqref{def:sgd_w} and \eqref{def:sgd_a}. 
In the following, we set $b_k := (k+1)\cdot (k+1)!$.
For sequences or functions depending on a parameter $m$, we use the notation $O_m(\cdot)$.

In preparation, we define an additional macroscopic variable of neural networks with SGDs:
\begin{align}
    Q_{i, j}^m(t) = \plim_{n, d \to \infty} \frac{1}{d}\bm{w}_i^\top \bm{w}_j^t, ~~ i,j\in \{1,2,...,m\}, t = 1,2,....
\end{align}
$Q_{i, j}^m(t)$ corresponds to the overlap between $i$-th and $j$-th feature vectors.
Note that for a case with $i=j$, $Q_{i, j}^m(t) = 1$ holds because of the normalization step.

\begin{proposition}[Difference equation of macroscopic variable] \label{prop:differnce_equation}
There exists a sequence of macroscopic variables $\{R^m(t), Q^m(t), a^m(t)\}_{t \in \N}$ such that for any $i,j \in \{1,...,m\}$ we have
\begin{align}
    R^m(t) = R_i^m(t),\quad
    Q^m(t) = Q_{i, j}^m(t)\ (i \ne j),\quad 
    a^m(t) = a_i^m(t).
\end{align}
Further, for any $t \in \N \cup \{0\}$, they satisfy the following a recursive system
    \begin{align}  \label{eq:discrete} 
    \begin{aligned}        
        R^m(t + 1) &= R^m(t) + \gamma m^{-1} \Big\{ a^m(t) \big( 1 - R^m(t)^2 \big) \sum_{k=0}^{\infty} b_k c_{\star, k+1} c_{k+1} R^m(t)^k \\
        &\quad- a^m(t)^2 R^m(t) \big( 1 - Q^m(t) \big) \sum_{k=0}^{\infty} b_k c_{k+1}^2 Q^m(t)^k\Big\} + O_m(m^{-2}),\\      
        Q^m(t + 1) &= Q^m(t) + \gamma m^{-1} \Big\{ 2a^m(t) R^m(t) \big( 1 - Q^m(t) \big)\sum_{k=0}^{\infty} b_k c_{\star, k+1}c_{k+1}R^m(t)^k \\
        &\quad -2 a^m(t)^2 Q^m(t) \big( 1 - Q^m(t) \big)\sum_{k=0}^{\infty} b_k c_{k+1}^2 Q^m(t)^k \Big\} 
        + O_m(m^{-2}),  \\    
        a^m(t + 1) &= a^m(t) + \gamma m^{-1} \Big\{ \sum_{k=1}^\infty k! c_{\star, k}c_k R^m(t)^k - a^m(t) \sum_{k=1}^\infty k! c_k^2 Q^m(t)^k \Big\} + O_m(m^{-2}),
        \end{aligned}
    \end{align}
with initialization $R^m(0) = Q^m(0) = 0$ and $ a^m(0) = \bar{a} \ne 0$.
\end{proposition}

The proof and derivation process is given in Section \ref{sec:discrete}.
Note that Assumptions \ref{ass:link} and \ref{ass:activation} guarantee the absolute convergence of the infinite series that appear in the right-hand side of (\ref{eq:discrete}).

We have some implications of the difference equation. First, we can reduce the number of macroscopic variables, which consist of $O_m(m^2)$ components, that are intractable when $m\to\infty$, to only three variables under the symmetric initialization in Assumption \ref{ass:symmetric_init}. 
Details will be presented in Lemma \ref{lem:symmetric} in Appendix \ref{sec:discrete}.
Second, the macro variable following \eqref{eq:discrete} is updated by balancing (i) the values determined by the interaction between $\sigma_\star$ and $\sigma$ through $c_{\star,k}$ and $c_k$, and (ii) the values determined solely by $\sigma$ through $c_k$.

\subsection{Derivation and validation of the discrete system}\label{sec:discrete} 

In this section, we rigorously derive recursive equations \eqref{eq:discrete}, and prove Lemma \ref{lem:symmetric} at the same time. The actual derivation of the recursive dynamics is essentially based on Tensor Programs (\citet{yang2019scaling, yang2020tensor2, yang2020tensor3, yang2021tensor4}). 
Following their notation of Tensor Programs, for a collection of potentially random $d$-dimensional vectors $x^1, x^2,...,x^k \in \R^d$, we consider real-valued random variables  $Z^{x^1},...,Z^{x^k}$ such that a distribution of elements of $x^j$ will be shown to converge to a distribution of $Z^{x^j}$ as $d \to \infty$, that is, it holds that
\begin{align}
    \frac{1}{d} \sum_{j=1}^d \psi(x_j^1,...,x_j^k) \to \Ep[\psi(Z^{x^1},...,Z^{x^k})]
\end{align}
almost surely for every polynomially bounded $\psi:\R^k \to \R$.
Further, we can decompose the random variable as $Z^{x^j} = \hat{Z}^{x^j} + \dot{Z}^{x^j}$, where $\hat{Z}^{x^j}$ and $\dot{Z}^{x^j}$ has specific formulation (see Box 2 in \cite{yang2020tensor3}.
Detailed discussion and examples are given in \cite{yang2019scaling,yang2020tensor3}.

In the following, we proceed by induction: assuming the statement of Lemma~\ref{lem:symmetric} and the equation \eqref{eq:discrete} hold at the iteration  
$t$, we show they also hold for the iteration $t+1$. Then we obtain that Lemma \ref{lem:symmetric} and the recursive equation \eqref{eq:discrete} hold.

\textbf{Preparation.} As a preliminary, we introduce several functions and constants for convenience. In Section \ref{sec:discrete}, we abbreviate notation by writing $R(t),   Q(t),   a(t)$ instead of $R^m(t),   Q^m(t),   a^m(t)$. First, for functions \(a, b, c, d:\mathbb{R}\to \mathbb{R}\) that are integrable with respect to the Gaussian measure, define
\begin{align}
    &I_R(t; a, b) = \mathbb{E}[a(z_1)b(z_2)], \quad
    I_Q(t; a, b) = \mathbb{E}[a(z_3)b(z_4)], \\
    &J_R(t; a, b, c) = \mathbb{E}[a(z_5)b(z_6)c(z_7)], \quad
    J_Q(t; a, b, c) = \mathbb{E}[a(z_8)b(z_9)c(z_{10})], \\
    &K_R(t; a, b, c, d)= \mathbb{E}[a(z_{11})b(z_{12})c(z_{13})d(z_{14}))], \quad
    K_Q(t; a, b, c, d)= \mathbb{E}[a(z_{15})b(z_{16})c(z_{17})d(z_{18}))],
\end{align}
where
\begin{align}
    &\begin{pmatrix}
        z_1(t) \\
        z_2(t)
    \end{pmatrix}
    \sim
    GP\left(\bm{0}, 
    \begin{pmatrix}
        1 & R(t) \\
        R(t) & 1
    \end{pmatrix}\right), \quad
    \begin{pmatrix}
        z_3(t) \\
        z_4(t)
    \end{pmatrix}
    \sim
    GP\left(\bm{0}, 
    \begin{pmatrix}
        1 & Q(t) \\
        Q(t) & 1
    \end{pmatrix}\right), \\
    &\begin{pmatrix}
        z_5(t) \\
        z_6(t) \\
        z_7(t)
    \end{pmatrix}
    \sim
    GP\left(\bm{0}, 
    \begin{pmatrix}
        1 & R(t) & R(t) \\
        R(t) & 1 & Q(t) \\
        R(t) & Q(t) & 1
    \end{pmatrix}\right), \quad
    \begin{pmatrix}
        z_8(t) \\
        z_9(t) \\
        z_{10}(t)
    \end{pmatrix}
    \sim
    GP\left(\bm{0}, 
    \begin{pmatrix}
        1 & Q(t) & Q(t) \\
        Q(t) & 1 & Q(t) \\
        Q(t) & Q(t) & 1
    \end{pmatrix}\right), \\
    &\begin{pmatrix}
        z_{11}(t) \\
        z_{12}(t) \\
        z_{13}(t) \\
        z_{14}(t) 
    \end{pmatrix}
    \sim
    GP\left(\bm{0}, 
    \begin{pmatrix}
        1 & R(t) & R(t) & R(t) \\
        R(t) & 1 & Q(t) & Q(t) \\
        R(t) & Q(t) & 1 & Q(t) \\
        R(t) & Q(t) & Q(t) & 1 
    \end{pmatrix}
    \right), \quad
    \begin{pmatrix}
        z_{15}(t) \\
        z_{16}(t) \\
        z_{17}(t) \\
        z_{18}(t) 
    \end{pmatrix}
    \sim
    GP\left(\bm{0}, 
    \begin{pmatrix}
        1 & Q(t) & Q(t) & Q(t) \\
        Q(t) & 1 & Q(t) & Q(t) \\
        Q(t) & Q(t) & 1 & Q(t) \\
        Q(t) & Q(t) & Q(t) & 1 
    \end{pmatrix}
    \right).
\end{align}
Also, we further define the following constants:
\begin{align}
    &s_1 = \mathbb{E}[\sigma'(z)^2], \quad
    s_2 = \mathbb{E}[\sigma(z)\sigma''(z)], \quad
    s_3 = \mathbb{E}[\sigma(z)^2\sigma'(z)^2], \quad
    s_4 = \mathbb{E}[\sigma(z)^2], 
\end{align}
where $z\sim\mathcal{N}(0, 1)$. In what follows, we will derive the model \eqref{eq:discrete}, and prove Lemma \ref{lem:symmetric} at the same time.
\begin{lemma}[Symmetricity of macroscopic variables]\label{lem:symmetric}
    Under Assumption \ref{ass:symmetric_init}, we have 
    $
        R_i^m(t) = R_j^m(t), 
        Q_{i, j}^m(t) = Q_{i', j'}^m(t)$, and $
        a_i^m(t) = a_j^m(t)
    $
    for any integer $t\geq 1$ and $i,j, i', j' \in \{1,\dots,m\}$ that satisfy $\ i \ne j, i \ne j'$.
\end{lemma}
We proceed with our proof inductively: let the model \eqref{eq:discrete} and the statement of Lemma \ref{lem:symmetric} hold for the $t$-th iteration, and prove the same statement for the $t+1$-th iteration.
We first calculate
\begin{align}
    \bm{G}_{\bm{w}}^t &= \nabla_{\bm{W}^t} \frac{1}{2n} \sum_{s=1}^n \Big\{ y_s^t - \frac{1}{m} \sum_{i=1}^m a_i^t \sigma(\langle\bm{x}_s^t,\bm{w}_i^t\rangle / \sqrt{d})\Big\}^2 \\
    &= \nabla_{\bm{W}^t} \frac{1}{2n} \Big\| \bm{y}^t - \frac{1}{m} \sigma(\bm{X}^t\bm{W}^t) \bm{a}^t \Big\|_2^2 \\
    &= -\frac{1}{n} \cdot\frac{1}{m} \bm{X}^{t \top} \Big\{ \big( \bm{y}^t \bm{a}^{t\top} - \frac{1}{m} \sigma(\bm{X}^t\bm{W}^t)\bm{a}^t \bm{a}^{t\top} \big) \odot \sigma'(\bm{X}^t\bm{W}^t) \Big\} \in \mathbb{R}^{d\times m},
\end{align}
where we defined $\bm{X}^t = (\bm{x}_1^t, \dots, \bm{x}_n^t)^\top / \sqrt{d}\in\mathbb{R}^{n\times d},  \bm{a}^t = (a_1^t,\dots, a_m^t)^\top \in\mathbb{R}^m$.  We introduce
\begin{align}
    \bm{\ell}_t = \big( \bm{y}^t \bm{a}^{t\top} - \frac{1}{m} \sigma(\bm{X}^t\bm{W}^t)\bm{a}^t \bm{a}^{t\top} \big) \odot \sigma'(\bm{X}^t\bm{W}^t) 
    =: (\bm{\ell}_{t, 1}, \dots, \bm{\ell}_{t, m})\in\mathbb{R}^{d\times m}
\end{align}
and express 
\begin{align}
    \bm{G}_{\bm{w}}^t = -\frac{1}{n}\cdot \frac{1}{m} \bm{X}^{t\top}\bm{\ell}_t 
    =: (\bm{G}_{\bm{w}, 1}^t , \dots, \bm{G}_{\bm{w}, m}^t) \in\mathbb{R}^{d\times m}.
\end{align}

\textbf{About $R(t)$:}
As the first step, we study the alignment term $R(t)$ through the analysis of $R_i(t)$.
Since we have 
\begin{align}
    R_i(t) = \Ep[Z^{\bm{w}^\star} Z^{\bm{w}_i^{t+1}}],    
\end{align}
we mainly study the term $Z^{\bm{w}_i^{t+1}}$.
To study this term, we recall the normalization step on the first layer: 
\begin{align}
    \bm{w}_i^{t + 1} = \frac{\sqrt{d}\tilde{\bm{w}}_i^{t+1}}{\| \tilde{\bm{w}}_i^{t+1} \|_2} = \frac{\tilde{\bm{w}}_i^{t+1}}{\sqrt{\tilde{\bm{w}}_i^{t+1 \top} \tilde{\bm{w}}_i^{t+1} / d}},
\quad
    \tilde{\bm{w}}_i^{t+1} = \bm{w}_i^t - \gamma d \bm{G}_{\bm{w}, i}^t,
\end{align}
Using the Tensor Programs formalism \citep{yang2019scaling}, we obtain the form
\begin{align}
    Z^{\bm{w}_i^{t+1}} &=Z^{\tilde{\bm{w}}_i^{t+1}} / \sqrt{\mathbb{E}[(Z^{\tilde{\bm{w}}_i^{t+1}})^2]}\\
    &= (Z^{\bm{w}_i^t} + \gamma Z^{-d\bm{G}_{\bm{w}, i}^t}) / \sqrt{\mathbb{E}[(Z^{\bm{w}_i^t} + \gamma Z^{-d \bm{G}_{\bm{w}, i}^t})^2]}. \label{eq:zw}
\end{align}
Then, we will study the term $Z^{-d\bm{G}_{\bm{w}, i}^t}$ and the expectation $\mathbb{E}[(Z^{\bm{w}_i^t} + \gamma Z^{-d \bm{G}_{\bm{w}, i}^t})^2]$.

First, we directly study the term $Z^{-d\bm{G}_{\bm{w}, i}^t}$.
By the variable decomposition of the Tensor Programs, (e.g.,  Box 1 in \cite{yang2020tensor3}), each element of $-d \bm{G}_{\bm{w}, i}^t$ asymptotically follows $Z^{-d \bm{G}_{\bm{w}, i}^t} $, which is decomposed as 
\begin{align}
    Z^{-d \bm{G}_{\bm{w}, i}^t} = \frac{1}{m \delta} \{ \hat{Z}^{\bm{X}^{t\top} \bm{\ell}_{t, i}} + \dot{Z}^{\bm{X}^{t\top} \bm{\ell}_{t, i}} \}, \label{eq:zdg}
\end{align}
where  the variables $\hat{Z}^{\bm{X}^{t\top} \bm{\ell}_{t, i}}$ and $\dot{Z}^{\bm{X}^{t\top} \bm{\ell}_{t, i}}$ follow Box 1 and Theorem 2.1 in \cite{yang2020tensor3}.
About the term $\dot{Z}^{\bm{X}^{t\top} \bm{\ell}_{t, i}}$ in \eqref{eq:zdg}, since we assume the statement of Lemma \ref{lem:symmetric} holds for the iteration $t$, we obtain the detailed form of $\dot{Z}^{\bm{X}^{t\top} \bm{\ell}_{t, i}}$ as
\begin{align}
    &\dot{Z}^{\bm{X}^{t \top}\bm{\ell}_{t, i}} \\
    &=  \delta a(t) \Big\{ \hat{Z}^{\bm{w}^\star} \mathbb{E}[\sigma_\star'(\hat{Z}^{\bm{X}^t \bm{w}^\star}) \sigma'(\hat{Z}^{\bm{X}^t \bm{w}_i^t})] \\
    &\quad +   Z^{\bm{w}_i^t} \mathbb{E}\Big[-\frac{a(t)}{m} \sigma'(\hat{Z}^{\bm{X}^t \bm{w}_i^t})^2 + \Big\{ \sigma_\star(\hat{Z}^{\bm{X}^t \bm{w}^\star}) + \hat{Z}^{\bm{\varepsilon}^t} - \frac{a(t)}{m}\sum_{j=1}^m \sigma(\hat{Z}^{\bm{X}^t \bm{w}_j^t})\Big\}  \sigma''(\hat{Z}^{\bm{X}^t \bm{w}_i^t}) \Big] \\
    &\quad +   \sum_{j\ne i}Z^{\bm{w}_j^t}\mathbb{E}\Big[-\frac{a(t)}{m} \sigma'(\hat{Z}^{\bm{X}^t \bm{w}_i^t}) \sigma'(\hat{Z}^{\bm{X}^t \bm{w}_j^t}) \Big]
    \Big\} \\
    &=   \delta a(t) \Big\{ \hat{Z}^{\bm{w}^\star} \mathbb{E}[\sigma_\star'(\hat{Z}^{\bm{X}^t \bm{w}^\star}) \sigma'(\hat{Z}^{\bm{X}^t \bm{w}_i^t})] \\
    &\quad +    Z^{\bm{w}_i^t} \Big( \mathbb{E}[\sigma_\star(\hat{Z}^{\bm{X}^t \bm{w}^\star})\sigma''(\hat{Z}^{\bm{X}^t \bm{w}_i^t})]
    - \frac{a(t)}{m}\sum_{j\ne i}\mathbb{E}[\sigma(\hat{Z}^{\bm{X}^t \bm{w}_j^t})\sigma''(\hat{Z}^{\bm{X}^t \bm{w}_i^t})] \\
    &\quad -   \frac{a(t)}{m} \mathbb{E}[\sigma'(\hat{Z}^{\bm{X}^t \bm{w}_i^t})^2] - \frac{a(t)}{m}\mathbb{E}[\sigma(\hat{Z}^{\bm{X}^t \bm{w}_i^t}) \sigma''(\hat{Z}^{\bm{X}^t \bm{w}_i^t})]
    \Big) \\
    &\quad -   \frac{a(t)}{m} \sum_{j\ne i} Z^{\bm{w}_j^t} \mathbb{E}[\sigma'(\hat{Z}^{\bm{X}^t \bm{w}_i^t}) \sigma'(\hat{Z}^{\bm{X}^t \bm{w}_j^t})]
    \Big\} \\
    &=   \delta a(t) \Big\{ \hat{Z}^{\bm{w}^\star} I_R(t; \sigma_\star', \sigma') 
    + Z^{\bm{w}_i^t} \Big( I_R(t; \sigma_\star, \sigma'') - a(t) \frac{m-1}{m} I_Q(t; \sigma, \sigma'') - \frac{a(t)}{m}s_1 - \frac{a(t)}{m}s_2 \Big) \\
    &\quad -   \frac{a(t)}{m} \sum_{j\ne i} Z^{\bm{w}_j^t}I_Q(t; \sigma', \sigma')
    \Big\}.
\end{align}
To achieve this form, we utilize the form
\begin{align}
    Z^{\bm{\ell}_{t, i}} = a_i^t  \Big\{ \sigma_\star(\hat{Z}^{\bm{X}^t \bm{w}^\star}) + \hat{Z}^{\bm{\varepsilon}^t} - \frac{1}{m}\sum_{j=1}^m a_j^t \sigma(\hat{Z}^{\bm{X}^t \bm{w}_j^t})\Big\} \sigma'(\hat{Z}^{\bm{X}^t \bm{w}_i^t}). \label{eq:zl}
\end{align}
Then, we can rewrite \eqref{eq:zdg} as
\begin{align} 
    &m Z^{-d \bm{G}_{\bm{w}, i}^t}  \\
    & = \frac{1}{\delta} (\hat{Z}^{\bm{X}^{t \top}\bm{\ell}_{t, i}} + \dot{Z}^{\bm{X}^{t \top}\bm{\ell}_{t, i}}) \\
    &=   \frac{1}{\delta} \hat{Z}^{\bm{X}^{t \top}\bm{\ell}_{t, i}}
    + a(t) \hat{Z}^{\bm{w}^\star} I_R(t;\sigma_\star', \sigma') \\
    & \quad +   Z^{\bm{w}_i^t} \Big\{ a(t) I_R(t; \sigma_\star, \sigma'') - a(t)^2 I_Q(t;\sigma, \sigma'') + \frac{a(t)^2}{m} \big(I_Q(t; \sigma, \sigma'') - s_1 - s_2\big) \\
    & \quad -   \frac{a(t)^2}{m} \sum_{j\ne i}Z^{\bm{w}_j^t} I_Q(t; \sigma', \sigma')
    \Big\}. \label{eq:zdg_new}
\end{align}

Second, we study the expectation $\mathbb{E}[(Z^{\bm{w}_i^t} + \gamma Z^{-d \bm{G}_{\bm{w}, i}^t})^2]$ in \eqref{eq:zw}, which includes the cross term $\frac{2\gamma}{m}\mathbb{E}[Z^{\bm{w}_i^t} m Z^{-d \bm{G}_{\bm{w}, i}^t}]$ and the second moments  $\mathbb{E}[(Z^{{\bm{w}}_i^{t+1}})^2]$ and $\mathbb{E}[(\frac{\gamma}{m}\cdot m Z^{-d \bm{G}_{\bm{w}, i}^t})^2] $.
In particular, we compute $\mathbb{E}[(Z^{\tilde{\bm{w}}_i^{t+1}})^2] = \lim_{n,d\to\infty} \| \tilde{\bm{w}}_i^{t+1} \|_2^2 / d$ following the relation $Z^{\tilde{\bm{w}}_i^t} = Z^{\bm{w}_i^t} + \gamma / m\cdot m Z^{-d \bm{G}_{\bm{w}, i}^t}$. 
About the cross term, we utilize \eqref{eq:zdg} and obtain
\begin{align}
    &\frac{2\gamma}{m}\mathbb{E}[Z^{\bm{w}_i^t} m Z^{-d \bm{G}_{\bm{w}, i}^t}] \\
    &=   \frac{2\gamma}{m} \Big\{ a(t) R(t) I_R(t; \sigma_\star', \sigma') + a(t) I_R(t; \sigma_\star, \sigma'') - a(t)^2 I_Q(t; \sigma, \sigma'') \\
    &\quad +   \frac{a(t)^2}{m} \big(I_Q(t; \sigma, \sigma'') - s_1 - s_2\big) - a(t)^2 \frac{m-1}{m} Q(t) I_Q(t; \sigma', \sigma')
    \Big\} \\
    &=   \frac{2\gamma}{m} \Big\{ a(t) R(t) I_R(t; \sigma_\star', \sigma') + a(t) I_R(t; \sigma_\star, \sigma'') - a(t)^2 I_Q(t; \sigma, \sigma'') - a(t)^2 Q(t) I_Q(t;\sigma', \sigma')
    \Big\} \\
    &\quad +  \frac{2\gamma}{m^2} \Big\{ a(t)^2 \big(I_Q(t; \sigma, \sigma'') - s_1 - s_2\big) + a(t)^2 Q(t) I_Q(t; \sigma', \sigma')
    \Big\} \\
    &=: \frac{2\gamma}{m} A(t) + O_m(m^{-2}).
\end{align}
Next, we calculate the second moments $\mathbb{E}[(\frac{\gamma}{m}\cdot m Z^{-d \bm{G}_{\bm{w}, i}^t})^2]$ by utilizing the decomposition \eqref{eq:zdg} as
\begin{align}
    \mathbb{E}\Big[\Big(\frac{\gamma}{m}\cdot m Z^{-d \bm{G}_{\bm{w}, i}^t}\Big)^2\Big] 
    = \frac{\gamma^2}{m^2} \mathbb{E}[(m Z^{-d \bm{G}_{\bm{w}, i}^t})^2] 
    = \frac{\gamma^2}{m^2} \Big\{\mathbb{E}\Big[\Big(\frac{1}{\delta} \hat{Z}^{\bm{X}^{t \top}\bm{\ell}_{t, i}}\Big)^2\Big] + \mathbb{E}\Big[\frac{1}{\delta}\Big(\dot{Z}^{\bm{X}^{t \top}\bm{\ell}_{t, i}}\Big)^2\Big]
    \Big\}.
\end{align}
With the relation $\mathbb{E}[(\frac{1}{\delta} \hat{Z}^{\bm{X}^{t \top}\bm{\ell}_{t, i}})^2] = \frac{1}{\delta}\mathbb{E}[( Z^{\bm{\ell}_{t, i}})^2]$ and the form \eqref{eq:zl}, it follows
\begin{align}
    &\mathbb{E}[( Z^{\bm{\ell}_{t, i}})^2] \\
    &=   a(t)^2 \mathbb{E}\Big[\Big\{ \sigma_\star(\hat{Z}^{\bm{X}^t \bm{w}^\star}) + \hat{Z}^{\epsilon^t} - \frac{a(t)}{m} \sum_{j=1}^m \sigma(\hat{Z}^{\bm{X}^t \bm{w}_j^t}) \Big\}^2 \sigma'(\hat{Z}^{\bm{X}^t \bm{w}_i^t})^2\Big] \\
    &=   a(t)^2   \mathbb{E}[ \sigma_\star( \hat{Z}^{\bm{X}^t \bm{w}^\star})^{2}
      \sigma'( \hat{Z}^{\bm{X}^t w^t_i} )^2]
    + \sigma_\varepsilon^2 a(t)^2   \mathbb{E}[ \sigma'( \hat{Z}^{\bm{X}^t w^t_i} )^2]
    \\ 
    & \quad - \frac{2 a(t)^3}{m} 
    \sum_{j \neq i} \mathbb{E}[ 
    \sigma_\star( \hat{Z}^{\bm{X}^t \bm{w}^\star} )
    \sigma( \hat{Z}^{\bm{X}^t w^t_j} )
    \sigma'( \hat{Z}^{\bm{X}^t w^t_i} )^{2}
    ]
    \\ 
    & \quad - \frac{2 a(t)^3}{m} 
    \mathbb{E}[
    \sigma_\star( \hat{Z}^{\bm{X}^t \bm{w}^\star} )
    \sigma( \hat{Z}^{\bm{X}^t w^t_i} )
    \sigma'( \hat{Z}^{\bm{X}^t w^t_i} )^{2}
    ]
    \\ 
    & \quad + \frac{a(t)^4}{m^2}
    \sum_{\substack{j=1 \\ j\neq i}}^{m}
    \sum_{\substack{k=1 \\ k\neq i,  k\neq j}}^{m}
    \mathbb{E}[
    \sigma( \hat{Z}^{\bm{X}^t w^t_j} )
    \sigma( \hat{Z}^{\bm{X}^t w^t_k} )
    \sigma'( \hat{Z}^{\bm{X}^t w^t_i} )^{2}
    ]
    \\ 
    & \quad  + \frac{a(t)^4}{m^2}
    \sum_{\substack{j=1 \\ j\neq i}}^{m}
    \mathbb{E}[
    \sigma( \hat{Z}^{\bm{X}^t w^t_j} )^{2}
    \sigma'( \hat{Z}^{\bm{X}^t w^t_i} )^{2}
    ]
     + \frac{a(t)^4}{m^2}
    \mathbb{E}[
    \sigma( \hat{Z}^{\bm{X}^t w^t_i} )^{2}
    \sigma'( \hat{Z}^{\bm{X}^t w^t_i} )^{2}
    ] \\
    &=   a(t)^2   I_R(t  ;  \sigma_{\star}^{2},  \sigma'^{ 2})
    + \sigma_\varepsilon^2 a(t)^2 s_{1}
    - 2 a(t)^3 \frac{m-1}{m}  
    J_R(t  ;  \sigma_{\star},  \sigma,  \sigma'^{ 2})
    \\ 
    & \quad - \frac{2 a(t)^3}{m}  
    I_R(t  ;  \sigma_{\star},  \sigma \cdot \sigma'^{ 2})
    + a(t)^4 \frac{(m-1)(m-2)}{m^{2}}  
    J_Q(t  ;  \sigma,  \sigma,  \sigma'^{ 2})
    \\ 
    & \quad + a(t)^4 \frac{m-1}{m^{2}}  
    I_Q(t  ;  \sigma^{2},  \sigma'^{ 2})
    + \frac{a(t)^4}{m^{2}}   s_{3} \\ 
    & =   a(t)^2 I_R(t  ;  \sigma_{\star}^{2},  \sigma'^{ 2})
    + a(t)^2 s_{1}
    - 2 a(t)^3   J_R(t  ;  \sigma_{\star},  \sigma,  \sigma'^{ 2})
    + a(t)^4   J_Q(t  ;  \sigma,  \sigma,  \sigma'^{ 2})
    \\ 
    & \quad  + \frac{1}{m}\Big\{ 2 a(t)^3   J_R(t  ;  \sigma_{\star},  \sigma,  \sigma'^{ 2})
    - 2 a(t)^3  
    I_R(t  ;  \sigma_{\star},  \sigma \cdot \sigma'^{ 2})
    \\ 
    &\quad - 3 a(t)^4  
    J_Q(t  ;  \sigma,  \sigma,  \sigma'^{ 2})
    + a(t)^4  
    I_Q(t  ;  \sigma^{2},  \sigma'^{ 2})\Big\}
    \\ 
    & \quad + \frac{1}{m^2}\Big\{ 2 a(t)^4  
    J_Q(t  ;  \sigma,  \sigma,  \sigma'^{ 2})
    - a(t)^4  
    I_Q(t  ;  \sigma^{2},  \sigma'^{ 2})
    + a(t)^4   s_{3} \Big\}\\
    &=   O_m (1). \label{eq:zl_2}
\end{align}
Also, we obtain
\begin{align}
    &\mathbb{E}\Big[
    \Big( \frac{1}{\delta} 
    \dot{Z}^{ \bm{X}^t \ell(\bm{X}^t \bm{w}^\star)_i}
    \Big)^{2}
    \Big] \\ 
    &=  
    \mathbb{E}\Big[ \Big\{
     a(t)  \hat{Z}^{ \bm{w}^\star} 
    I_R( t  ;  \sigma_{\star}',  \sigma' )
    + Z^{ w^{t}_{i}} 
    \Big(
    a(t)  I_R(t  ;  \sigma_{\star},  \sigma'' )
    -
    a(t)^{2}  I_Q(t  ;  \sigma,  \sigma'' )
    \\ 
    &\quad
    + \frac{a(t)^{2}}{m} 
    (
    I_Q(t  ;  \sigma,  \sigma'' )
    -
    s_{1} - s_{2}
    ) \Big)
    - \frac{a(t)^{2}}{m}
    \sum_{j \neq i}
    Z^{ \bm{w}_{j}} 
    I_Q(t  ;  \sigma',  \sigma' )
    \Big\}^2 \Big]
    \\ 
    &=   a(t)^{2}
    I_R(t  ;  \sigma_{\star}',  \sigma' )^{2}
    + \Big\{
    a(t)  I_R(t  ;  \sigma_{\star},  \sigma'' )
    -
    a(t)^{2} I_Q(t  ;  \sigma,  \sigma'' )
    + \frac{a(t)^{2}}{m}
    \big(
    I_Q(t  ;  \sigma,  \sigma'' )
    - s_{1} - s_{2}
    \big) \Big\}^{2}
    \\ 
    &\quad
    + \frac{a(t)^{4}}{m^{2}}
    \mathbb{E}\Big[
    \Big(
    \sum_{j\neq i} Z^{ \bm{w}_{j}} 
    I_Q(t  ;  \sigma',  \sigma' )
    \Big)^{2}
    \Big]
    \\ 
    &\quad
    + 2  a(t)  R(t) 
    I_R(t  ;  \sigma_{\star}',  \sigma' )
    \Big(
    a(t)  I_R(t  ;  \sigma_{\star},  \sigma'' )
    - a(t)^{2}  I_Q(t  ;  \sigma,  \sigma'' )
    + \frac{a(t)^{2}}{m}
    (
    I_Q(t  ;  \sigma,  \sigma'' )
    - s_{1} - s_{2}
    ) \Big)
    \\ 
    &\quad
    - 2  a(t)^{2}  \frac{m-1}{m} 
    Q(t)  I_Q(t  ;  \sigma',  \sigma' )
    \Big(
    a(t)  I_R(t  ;  \sigma_{\star},  \sigma'' )
    - a(t)^{2}  I_Q(t  ;  \sigma,  \sigma'' ) \\
    &\quad
    + \frac{a(t)^{2}}{m}
    (
    I_Q(t  ;  \sigma,  \sigma'' )
    - s_{1} - s_{2}
    ) \Big)
    - 2  a(t)^{3}  \frac{m-1}{m} 
    R(t) 
    I_R(t  ;  \sigma_{\star}',  \sigma' )
    I_Q(t  ;  \sigma',  \sigma' )
    \\ 
    &=   a(t)^{2}
    I_R(t  ;  \sigma_{\star}',  \sigma')^2
    + \Big(
    a(t)  I_R(t  ;  \sigma_{\star},  \sigma'' )
    - a(t)^{2} I_Q(t  ;  \sigma,  \sigma'' )
    + \frac{a(t)^{2}}{m}
    (
    I_Q(t  ;  \sigma,  \sigma'' )
    - s_{1} - s_{2}
    )^{2} \Big)
    \\ 
    &\quad
    + a(t)^{4}
    I_Q(t  ;  \sigma',  \sigma' )^{2}
    \left\{
    \frac{(m-1)(m-2)}{m^{2}}  Q(t)
    + \frac{m-1}{m^{2}}
    \right\}
    \\ 
    &\quad
    + 2  a(t) R(t)
    I_R(t  ;  \sigma_{\star}',  \sigma' )
    \Big(
    a(t)  I_R(t  ;  \sigma_{\star},  \sigma'' )
    - a(t)^{2} I_Q(t  ;  \sigma,  \sigma'' )
    + \frac{a(t)^{2}}{m}
    (
    I_Q(t  ;  \sigma,  \sigma'' )
    - s_{1} - s_{2}
    ) \Big)
    \\ 
    &\quad
    - 2  a(t)^{2} \frac{m-1}{m}
    Q(t) 
    I_Q(t  ;  \sigma',  \sigma' )
    \Big(
    a(t)  I_R(t  ;  \sigma_{\star},  \sigma'' )
    - a(t)^{2} I_Q(t  ;  \sigma,  \sigma'' ) \\
    &\quad
    + \frac{a(t)^{2}}{m}
    (
    I_Q(t  ;  \sigma,  \sigma'' ) 
    - s_{1} - s_{2}
    ) \Big)
    \\ 
    &\quad
    - 2  a(t)^{3} \frac{m-1}{m}
    R(t) 
    I_R(t  ;  \sigma_{\star}',  \sigma' )
    I_Q(t  ;  \sigma',  \sigma' )
    \\ 
    & =   a(t)^{2}
    I_R(t  ;  \sigma_{\star}',  \sigma' )^{2}
    + \Big\{
    a(t)  I_R(t  ;  \sigma_{\star},  \sigma'' )
    - a(t)^{2} I_Q(t  ;  \sigma,  \sigma'' )
    \Big\}^{2}
    + a(t)^{4}
    I_Q(t  ;  \sigma',  \sigma' )^{2}
      Q(t)
    \\[12pt]
    &\quad
    + 2  a(t) R(t) 
    I_R(t  ;  \sigma_{\star}',  \sigma' )
    \Big\{
    a(t)  I_R(t  ;  \sigma_{\star},  \sigma'' )
    - a(t)^{2} I_Q(t  ;  \sigma,  \sigma'' )
    \Big\}
    \\ 
    &\quad
    - 2  a(t)^{2} Q(t) 
    I_Q(t  ;  \sigma',  \sigma' )
    \Big\{
    a(t)  I_R(t  ;  \sigma_{\star},  \sigma'' )
    - a(t)^{2} I_Q(t  ;  \sigma,  \sigma'' )
    \Big\}
    \\ 
    &\quad
    - 2  a(t)^{3} R(t) 
    I_R(t  ;  \sigma_{\star}',  \sigma' )
    I_Q(t  ;  \sigma',  \sigma' )
    \\[12pt]
    &\quad
    + m^{-1}\Big[
    2 a(t)^{2}
    \big(
    I_Q(t  ;  \sigma, \sigma'') - s_{1} - s_{2}
    \big)
    \big(
    a(t) I_R(t  ;  \sigma_{\star}, \sigma'')
    - a(t)^{2} I_Q(t  ;  \sigma, \sigma'')
    \big)
    \\ 
    &\quad
    + a(t)^{4}
    I_Q(t  ;  \sigma', \sigma')^{2}
    \big(
    -3 Q(t) + 1
    \big)
    + 2 a(t)^{3} R(t)
    I_R(t  ;  \sigma_{\star}', \sigma')
    \big(
    I_Q(t  ;  \sigma, \sigma'')
    - s_{1} - s_{2}
    \big)
    \\ 
    &\quad
    - 2 a(t)^{2} Q(t) 
    I_Q(t  ;  \sigma', \sigma')
    \Big\{
    a(t)^{2}
    \big(
    I_Q(t  ;  \sigma, \sigma'')
    - s_{1} - s_{2}
    \big)
    - a(t) I_R(t  ;  \sigma_{\star}, \sigma'')
    + a(t)^{2} I_Q(t  ;  \sigma, \sigma'')
    \Big\}
    \\ 
    &\quad
    + 2 a(t)^{3} R(t)
    I_R(t  ;  \sigma_{\star}', \sigma')
    I_Q(t  ;  \sigma', \sigma') \Big]
    \\[12pt]
    &\quad
    + m^{-2}
    \Big\{
    a(t)^{4}
    \big(
    I_Q(t  ;  \sigma, \sigma'')
    - s_{1} - s_{2}
    \big)^2
    + a(t)^{4}
    I_Q(t  ;  \sigma', \sigma')^{2}
    \big(
    2 Q(t) - 1
    \big)
    \\ 
    &\quad
    + 2 a(t)^{4} Q(t) 
    I_Q(t  ;  \sigma', \sigma')
    \big(
    I_Q(t  ;  \sigma, \sigma'')
    - s_{1} - s_{2}
    \big) \Big\} \\
    &=   O_m (1). \label{eq:zxl_2}
\end{align}
Combining the results \eqref{eq:zl_2} and \eqref{eq:zxl_2} for $\mathbb{E}[(\frac{1}{\delta} \hat{Z}^{\bm{X}^{t \top}\bm{\ell}_{t, i}})^2]$ and $\mathbb{E}[
    ( \frac{1}{\delta} 
    \dot{Z}^{ \bm{X}^t \ell(\bm{X}^t \bm{w}^\star)_i}
    )^{2}]$ with the relation \eqref{eq:zdg_new}, we get
\begin{align}
    \mathbb{E}\Big[\Big(\frac{\gamma}{m}\cdot m Z^{-d \bm{G}_{\bm{w}, i}^t}\Big)^2\Big] = O_m(m^{-2}).
\end{align}
Thus, it follows that
\begin{align}
    \mathbb{E}[(Z^{\tilde{\bm{w}}_i^{t+1}})^2] &= \mathbb{E}[(Z^{\bm{w}_i^t} + \gamma Z^{-d \bm{G}_{\bm{w}, i}^t})^2] \\ 
    &=\ \mathbb{E}[(Z^{\bm{w}_i^t} + \gamma m^{-1} \cdot m Z^{-d \bm{G}_{\bm{w}, i}^t})^2]  \\
    &= 1 + 2 A(t) \gamma m^{-1} + O_m(m^{-2}),
\end{align}
then we evaluate the expectation term in \eqref{eq:zw} as
\begin{align}
    \{\mathbb{E}[(Z^{\bm{w}_i^t} + \gamma Z^{-d \bm{G}_{\bm{w}, i}^t})^2]\}^{-1/2} 
    =\ 1 - A(t) \gamma m^{-1} + O_m(m^{-2}). \label{eq:variance}
\end{align}

Now, we are ready to study $R_i(t)$.
By using \eqref{eq:variance}, we update \eqref{eq:variance} as
\begin{align}
    Z^{\bm{w}_i^{t+1}} &=Z^{\tilde{\bm{w}}_i^{t+1}} / \sqrt{\mathbb{E}[(Z^{\tilde{\bm{w}}_i^{t+1}})^2]}\\
    &= (Z^{\bm{w}_i^t} + \gamma Z^{-d\bm{G}_{\bm{w}, i}^t}) / \sqrt{\mathbb{E}[(Z^{\bm{w}_i^t} + \gamma Z^{-d \bm{G}_{\bm{w}, i}^t})^2]} \\
    &= (Z^{\bm{w}_i^t} + \gamma Z^{-d\bm{G}_{\bm{w}, i}^t}) \times \{ 1 - A(t) \gamma m^{-1} + O_m(m^{-2})\}.
\end{align}
Therefore, by multiplying $Z^{\bm{w}^\star}$ on both sides and taking expectation, we derive 
\begin{align}
    R_i(t+1) &= \big(R(t) + \gamma \mathbb{E}[Z^{\bm{w}^\star} Z^{-d\bm{G}_{\bm{w}, i}^t}]\big) \times \{ 1 - A(t) \gamma m^{-1} + O_m(m^{-2}) \} \\
    &= \big(R(t) + \gamma m^{-1} \mathbb{E}[Z^{\bm{w}^\star} mZ^{-d\bm{G}_{\bm{w}, i}^t}]\big) \times \{ 1 - A(t) \gamma m^{-1} + O_m(m^{-2}) \}.
\end{align}
The appearing cross term is also evaluate as using \eqref{eq:zdg_new} as
\begin{align}
    &\mathbb{E}[Z^{\bm{w}^\star} mZ^{-d\bm{G}_{\bm{w}, i}^t}] 
    \\ 
    &=   a(t) 
    I_R(t  ;  \sigma_{\star}',  \sigma' )
    + R(t)
    \Big\{
    a(t)  I_R(t  ;  \sigma_{\star},  \sigma'' )
    - a(t)^{2} I_Q(t  ;  \sigma,  \sigma'' )
    + \frac{a(t)^{2}}{m}
    \big(
    I_Q(t  ;  \sigma,  \sigma'' )
    - s_{1} - s_{2}
    \big)
    \Big\}
    \\ 
    &\quad
    - a(t)^{2}  \frac{m-1}{m} 
    R(t) 
    I_Q(t  ;  \sigma',  \sigma' )
    \\ 
    &=   a(t) 
    I_R(t  ;  \sigma_{\star}',  \sigma' )
    + R(t)
    \Big\{
    a(t)  I_R(t  ;  \sigma_{\star},  \sigma'' )
    - a(t)^{2} I_Q(t  ;  \sigma,  \sigma'' )
    \Big\}
    - a(t)^{2} R(t) 
    I_Q(t  ;  \sigma',  \sigma' )
    \\ 
    &\quad
    + \frac{1}{m}\Big\{ a(t)^{2} 
    R(t)
    \big(
    I_Q(t  ;  \sigma,  \sigma'' )
    - s_{1} - s_{2}
    \big)
    + a(t)^{2} 
    R(t) 
    I_Q(t  ;  \sigma',  \sigma' )\Big\}
    \\ 
    &=:   B(t) + O_m (m^{-1}),
\end{align}
and we have
\begin{align}
    R_i(t+1) &= \{ R(t) + B(t) m^{-1} + O_m(m^{-2}) \} \times \{ 1 - A(t) \gamma m^{-1} + O_m(m^{-2}) \} \\
    &= R(t) + \gamma m^{-1} \{ -R(t) A(t) + B(t) \} + O_m(m^{-2}),
\end{align}
where the right hand side doesn't depend on $i$, so we can simply write $R(t+1)$, instead of $R_i(t+1)$. Here, we finally obtain
\begin{align}
    R(t+1) = R(t) + \gamma m^{-1} \{ -R(t) A(t) + B(t) \} + O_m(m^{-2}),
\end{align}
where $-R(t) A(t) + B(t)$ is equal to the expression of the model \eqref{eq:discrete}.

\textbf{About $Q(t)$:}
The derivation of the equation for $Q(t)$ proceeds quite similarly. Just like the above, we obtain\begin{align}
    &Q_{i, j}(t+1) \\
    &= \big\{ Q(t) + \gamma m^{-1}\mathbb{E}[Z^{\bm{w}_j^t} m Z^{-d\bm{G}_{\bm{w}, i}^t}] + \gamma m^{-1}\mathbb{E}[Z^{\bm{w}_i^t} m Z^{-d\bm{G}_{\bm{w}, j}^t}] + \gamma^2 m^{-2} \mathbb{E}[m Z^{-d\bm{G}_{\bm{w}, i}^t} m Z^{-d\bm{G}_{\bm{w}, j}^t}] \big\} \\ 
    & \quad \times \{ 1 - A(t) \gamma m^{-1} + O_m(m^{-2}) \}^2 \qquad (i \ne j),
\end{align}
where we can easily check $\mathbb{E}[
    Z^{\bm{w}_j^t} 
    m Z^{-d\bm{G}_{i}^{t}}
    ] = \mathbb{E}[
    Z^{\bm{w}_i^t} 
    m Z^{-d\bm{G}_{j}^{t}}
    ]$. We obtain
\begin{align}
    &\mathbb{E}[
    Z^{\bm{w}_j^t} 
    m Z^{-d\bm{G}_{i}^{t}}
    ] \\
    &=  
    a(t) R(t) 
    I_R(t  ;  \sigma_{\star}',  \sigma' )
    \\
    & \quad + Q(t)
    \Big\{
    a(t)  I_R(t  ;  \sigma_{\star},  \sigma'' )
    - a(t)^{2}  I_Q(t  ;  \sigma,  \sigma'' )
    + \frac{a(t)^{2}}{m}
    \big(
    I_Q(t  ;  \sigma,  \sigma'' )
    - s_{1} - s_{2}
    \big)
    \Big\}
    \\ 
    &\quad
    - a(t)^{2}  \frac{m-2}{m} 
    Q(t) 
    I_Q(t  ;  \sigma',  \sigma' )
    - \frac{a(t)^{2}}{m} 
    I_Q(t  ;  \sigma',  \sigma' ) \\ 
    &=  
    a(t) R(t) 
    I_R(t  ;  \sigma_{\star}',  \sigma' )
    + Q(t)
    \Big\{
    a(t)  I_R(t  ;  \sigma_{\star},  \sigma'' )
    - a(t)^{2} I_Q(t  ;  \sigma,  \sigma'' )
    \Big\}
    - a(t)^{2} Q(t) 
    I_Q(t  ;  \sigma',  \sigma' )
    \\ 
    &\quad
    + \frac{1}{m}\Big\{ a(t)^{2} 
    Q(t)
    \big(
    I_Q(t  ;  \sigma,  \sigma'' )
    - s_{1} - s_{2}
    \big)
    + 2 a(t)^{2} 
    Q(t) 
    I_Q(t  ;  \sigma',  \sigma' )
    - a(t)^{2} 
    I_Q(t  ;  \sigma',  \sigma' )\Big\}
    \\ 
    &= : \frac{1}{2} C(t) + O_m (m^{-1}).
\end{align}
The cross term can be decomposed as:
\begin{align}
    \mathbb{E}[
    m Z^{-d\bm{G}_{\bm{w}, i}^t}
      m Z^{-d\bm{G}_{\bm{w}, j}^t}
    ]
    =
    \frac{1}{\delta^{2}} 
    \mathbb{E}[
    \hat{Z}^{ \bm{X}^{t \top}\bm{\ell}_{t,i}}
    \hat{Z}^{ \bm{X}^{t \top}\bm{\ell}_{t,j}}
    ]
    + \frac{1}{\delta^{2}} 
    \mathbb{E}[
    \dot{Z}^{ \bm{X}^{t \top}\bm{\ell}_{t,i}}
    \dot{Z}^{ \bm{X}^{t \top}\bm{\ell}_{t,j}}
    ].
\end{align}
The first term can be calculated as:
\begin{align}
    &\frac{1}{\delta^{2}}
    \mathbb{E}[
    \hat{Z}^{ \bm{X}^{t \top}\bm{\ell}_{t,i}}
    \hat{Z}^{ \bm{X}^{t \top}\bm{\ell}_{t,j}}
    ]
    \\
    &=  
    \frac{a(t)^{2}}{\delta} 
    \mathbb{E}\Big[
    \Big\{
    \sigma_{\star}(\hat{Z}^{ \bm{X}^tw^{\star}})
    + \hat{Z}^{ \varepsilon^{t}}
    - \frac{a(t)}{m}
    \sum_{k=1}^{m}
    \sigma(\hat{Z}^{ \bm{X}^tw_{k}^{t}})
    \Big\}^2
    \sigma'(\hat{Z}^{ \bm{X}^tw_{i}^{t}})
    \sigma'(\hat{Z}^{ \bm{X}^tw_{j}^{t}})
    \Big]
    \\[14pt]
    &=  
    \frac{a(t)^{2}}{\delta}
    \Big\{
    J_R(t  ;  \sigma_{\star}^{2},  \sigma',  \sigma' )
    + I_Q(t  ;  \sigma',  \sigma' )
    + a(t)^{2}  K_Q(t  ;  \sigma,  \sigma,  \sigma',  \sigma' )
    - 2 a(t) 
    K_R(t  ;  \sigma_{\star},  \sigma,  \sigma',  \sigma' )
    \Big\} \\ 
    & \quad +   O_m (m^{-1}) 
    =O_m (1).
\end{align}
Similarly, for the second term:
\begin{align}
    &\mathbb{E}\Big[
    \frac{1}{\delta} 
    \dot{Z}^{ \bm{X}^{t \top}\bm{\ell}_{t,i}}
    \frac{1}{\delta} 
    \dot{Z}^{ \bm{X}^{t \top}\bm{\ell}_{t,j}}
    \Big]\\ 
    &=  
    a(t)^{2} 
    I_R(t  ;  \sigma_{\star}',  \sigma')^2
    \\ 
    &\quad
    + Q(t)
    \Big\{
    a(t)  I_R(t  ;  \sigma_{\star},  \sigma'' )
    - a(t)^{2} I_Q(t  ;  \sigma,  \sigma'' )
    + \frac{a(t)^{2}}{m} 
    \big(
    I_Q(t  ;  \sigma,  \sigma'' )
    - s_{1} - s_{2}
    \big)
    \Big\}^{2}
    \\ 
    &\quad
    + a(t)^{4} Q(t) 
    I_Q(t  ;  \sigma',  \sigma' )^{2}
    \\ 
    &\quad
    + 2  a(t) R(t) 
    I_R(t  ;  \sigma_{\star}',  \sigma')
    \\
    &\qquad\quad\times
    \Big\{
    a(t)  I_R(t  ;  \sigma_{\star},  \sigma'' )
    - a(t)^{2}  I_Q(t  ;  \sigma,  \sigma'' )
    + \frac{a(t)^{2}}{m}
    \big(
    I_Q(t  ;  \sigma,  \sigma'' )
    - s_{1} - s_{2}
    \big)
    \Big\}
    \\ 
    &\quad
    - 2  a(t)^{2} Q(t) 
    I_Q(t  ;  \sigma',  \sigma' )
    \\
    &\qquad\quad\times
    \Big\{
    a(t)  I_R(t  ;  \sigma_{\star},  \sigma'' )
    - a(t)^{2}  I_Q(t  ;  \sigma,  \sigma'' )
    + \frac{a(t)^{2}}{m}
    \big(
    I_Q(t  ;  \sigma,  \sigma'' )
    - s_{1} - s_{2}
    \big)
    \Big\}
    \\ 
    &\quad
    - 2  a(t)^{3} R(t) 
    I_R(t  ;  \sigma_{\star}',  \sigma')
    I_Q(t  ;  \sigma',  \sigma' ) + O_m (m^{-1}) \\
    & =    O_m (1).
\end{align}
Combining these results, the equation of $Q$ is reduced to the following:
\begin{align}
    Q_{i, j}(t+1) =  & \big\{ Q(t) + C(t) \gamma m^{-1} + O_m(m^{-2}) \big\} \times \{ 1 - A(t) \gamma m^{-1} + O_m(m^{-2}) \}^2 \\ 
    =  &
    Q(t)
    + \big\{ -2A(t) Q(t) + C(t) \big\}\gamma m^{-1} + O_m(m^{-2}).
\end{align}
We observe that $Q_{i, j}(t+1)$ does not depend on $i,  j$, so we simply write it $Q(t+1)$, and obtain
\begin{align}
    Q(t+1) = Q(t)
    + \big\{ -2A(t) Q(t) + C(t) \big\}\gamma m^{-1} + O_m(m^{-2}),
\end{align}
where $-2A(t) Q(t) + C(t)$ is the same as the form of \eqref{eq:discrete}. 

\textbf{About $a(t)$:}
As a final step, we derive the equation for $a(t)$. Because there is no normalization step for the second layer updates, the calculation is much simpler than that of $R(t)$ or $Q(t)$. Let $\bm{G}_{\bm{a}}^t = \nabla_{\bm{a}} \| \bm{y}^t - \frac{1}{m}\sum_{j=1}^m a_j^t \sigma(\bm{X}^t \bm{w}_j^t) \|_2^2$, then the second layer update proceeds as follows:
\begin{align}
    \bm{a}^{t+1} &= \bm{a}^t - \gamma \bm{G}_{\bm{a}}^t \\
    &= \bm{a}^t - \gamma m^{-1} \cdot m\bm{G}_{\bm{a}}^t \in\mathbb{R}^m,
\end{align}
where $\bm{G}_{\bm{a}}^t := (\bm{G}_{\bm{a}, 1}^t, \dots, \bm{G}_{\bm{a}, m}^t)$. For each $\bm{G}_{\bm{a}, i}^t$, one has
\begin{align}
    m \bm{G}_{\bm{a}, i}^t &= -\frac{1}{n} \sigma(\bm{X}^t \bm{w}_i^t)^\top \{\sigma_\star(\bm{X}^t \bm{w}^\star) + \bm{\varepsilon}^t  - \frac{a(t)}{m}\sum_{j=1}^m \sigma(\bm{X}^t \bm{w}_j^t)\} \\ 
    &= -\frac{1}{n} \sigma(\bm{X}^t \bm{w}_i^t)^\top \sigma_\star(\bm{X}^t \bm{w}^\star) + \frac{a(t)}{m} \sum_{j=1}^m \frac{1}{n} \sigma(\bm{X}^t \bm{w}_i^t)^\top \sigma(\bm{X}^t \bm{w}_j^t) \\ 
    \overset{n, d \to \infty}{\longrightarrow}& - \mathbb{E}[\sigma(\hat{Z}^{\bm{X}^t \bm{w}_i^t}) \sigma_\star(\hat{Z}^{\bm{X}^t \bm{w}_\star})] + a(t) \frac{m-1}{m}\mathbb{E}[\sigma(\hat{Z}^{\bm{X}^t \bm{w}_i^t})\sigma(\hat{Z}^{\bm{X}^t \bm{w}_j^t})] + \frac{a(t)}{m} \mathbb{E}[\sigma(\hat{Z}^{\bm{X}^t \bm{w}_i^t})^2] \\ 
    &= - I_R(t; \sigma, \sigma_\star) + a(t) I_Q(t; \sigma, \sigma) -m^{-1} a(t) I_Q(t; \sigma, \sigma) + m^{-1} a(t) s_4. 
\end{align}
Since this form is independent of $i$, we can write
\begin{align}
    a(t+1) &= a(t) + \gamma m^{-1} \big\{  I_R(t; \sigma, \sigma_\star) - a(t) I_Q(t; \sigma, \sigma)\big\} + \gamma m^{-2} a(t) \big\{ I_Q(t; \sigma, \sigma) - s_4 \big\},
\end{align}
where we can see this matches the model \eqref{eq:discrete} by Hermite expansion. This completes the proof.
\qed

\subsection{Proof of Proposition \ref{prop:ode_validation}}

This limiting ODE is directly derived from the standard Euler method.
Here, by using Lemma \ref{lem:Q=R^2} (presented in Section \ref{sec:proof_Q=R^2}), we utilize the relation $R_\tau^2 = Q_\tau$, we can omit the variable $Q_\tau$.
Then, we can formulate it as the following lemma.
\begin{lemma}\label{lem:euler}
    Let $R_\tau^m := R^m(\lfloor m \tau / \gamma \rfloor),   a_\tau^m := a^m(\lfloor m \tau / \gamma\rfloor)$. Then, for any finite $\tau \ge 0$, asymptotic equalities
    \begin{align}
        \lim_{m\to\infty} R_\tau^m = R_\tau, \quad \lim_{m\to\infty} a_\tau^m = a_\tau
    \end{align}
    hold for $R_\tau, a_\tau$ satisfying the ODE \eqref{eq:original}.
\end{lemma}
\begin{proof}[Proof of Lemma \ref{lem:euler}]
    Since $f$ and $g$ in \eqref{eq:original} are analytic functions, the result immediately holds from the standard discussion of numerical analysis, e.g., Theorem 2.4 in \citet{atkinson2009numerical}.
\end{proof}

Finally, by combining the results, we can prove Proposition \ref{prop:ode_validation}:
\begin{proof}[Proof of Proposition \ref{prop:ode_validation}]
    It immediately holds by Proposition \ref{prop:differnce_equation} and Lemma \ref{lem:euler}.
\end{proof}

\subsection{Reduction of $Q_\tau$}\label{sec:proof_Q=R^2}

We provide the following lemma, which allows us to omit the variable $Q_\tau$.
\begin{lemma}\label{lem:Q=R^2}
    For any $\tau \ge 0$, it holds that
    \begin{align}
        Q_\tau = R_\tau^2.
    \end{align}
\end{lemma}
\begin{proof}
    Let $U_\tau = Q_\tau - R_\tau^2$. Then, it follows that
    \begin{align}
        \frac{dU_\tau}{d\tau} &= \frac{dQ_\tau}{d\tau} - 2 R_\tau \frac{dR_\tau}{d\tau} \\
        &= 2 a_\tau R_\tau \big( R_\tau^2 - Q_\tau \big) \sum_{k=0}^{\infty} (k+1)\cdot (k+1)! c_{\star, k+1} c_{k+1} R_\tau^k \\
        &\quad - 2 a_\tau^2 \big( R_\tau^2 - Q_\tau \big) \big( 1 - Q_\tau \big) \sum_{k=0}^{\infty} (k+1)\cdot (k+1)! c_{k+1}^2 Q_\tau^k \\
        &= 2 a_\tau \big( R_\tau^2 - Q_\tau \big) \Big\{ R_\tau \sum_{k=0}^{\infty} (k+1)\cdot (k+1)! c_{\star, k+1} c_{k+1} R_\tau^k \\
        &\quad + a_\tau \big( 1 - Q_\tau \big) \sum_{k=0}^{\infty} (k+1)\cdot (k+1)! c_{k+1}^2 Q_\tau^k\Big\} \\
        &= -2 U_\tau V_\tau,
    \end{align}
    where
    \begin{align}
        V_\tau &= a_\tau \Big\{ R_\tau \sum_{k=0}^{\infty} (k+1)\cdot (k+1)! c_{\star, k+1} c_{k+1} R_\tau^k + a_\tau \big( 1 - Q_\tau \big) \sum_{k=0}^{\infty} (k+1)\cdot (k+1)! c_{k+1}^2 Q_\tau^k\Big\}.
    \end{align}
    Then we obtain
    \begin{align}
        U_\tau = U_0 \exp\Big( -2\int_0^\tau V_s ds \Big) = 0,
    \end{align}
    since $U_0 = 0$.
\end{proof}

\section{Derivation of ODE via population gradient} \label{sec:alt_derivation}
We show that we can derive the same ODE as \eqref{eq:original} by leveraging the gradient flow of the population loss. 
In particular, we consider an expected version of a loss with the two-layer neural network $
f(\bm x;{\bm a},{\bm W})
= \frac{1}{m}\sum_{i=1}^m {a}_{, i}  
\sigma(\langle{\bm w}_{i},\bm x\rangle / \sqrt d)$, and define a solution of an ODE defined by the gradient of the expected loss.

For time $\tau > 0$, we define a solution $(\check{\bm a}_\tau, \check{\bm W}_\tau)_{\tau \geq 0}$ with initialization
\begin{align}
\check{a}_{i, 0}=\bar a, \quad 
\check{\bm w}_{i, 0}, \, \check{\bm w}_\star \stackrel{\mathrm{i.i.d.}}{\sim} \mathrm{Unif}\left(\mathbb S^{d-1}(\sqrt d)\right),
\end{align}
and the gradient of the population loss:
\begin{align}
\mathcal L(\check{\bm a}_\tau,\check{\bm W}_\tau)
=& \frac12 \mathbb{E}\Big[
\big(
\sigma(\langle \check{\bm w}_\star,\bm x\rangle / \sqrt d)
- \frac{1}{m}\sum_{i=1}^m \check a_{i, \tau}  
\sigma(\langle \check{\bm w}_{i, \tau},\bm x\rangle / \sqrt d)
\big)^2
\Big] \\
=& \frac12\Big(
\frac{1}{m^2}\sum_{i,j=1}^m \check a_{i, \tau} \check a_{j, \tau}
 Y(\langle \check{\bm w}_{i, \tau},\check{\bm w}_{j, \tau}\rangle / d)
- \frac{2}{m}\sum_i \check a_{i, \tau} 
 S(\langle \check{\bm w}_\star,\bm w_{i, \tau}\rangle / d)
\Big)
+ \mathrm{const},
\end{align}
with
\begin{align}   
    S(z)=\sum_{k=1}^\infty c_{\star, k} c_k z^k,
    \quad
    Y(z)=\sum_{k=1}^\infty c_k^2 z^k.
\end{align}
We then calculate the gradients as follows:
\begin{align}
    \frac{d \check a_{i, \tau}}{d\tau}
    &= -m \partial_{\check a_{i, \tau}}\mathcal L (\check{\bm a}_\tau,\check{\bm W})
    = -\frac1m\sum_{j=1}^m \check a_{j, \tau}
    Y(\langle \check{\bm w}_i,\check{\bm w}_j\rangle / d)
    + S(\langle \check{\bm w}_\star,\check{\bm w}_i\rangle / d), \\
    \frac{d \check{\bm w}_{i, \tau}}{d\tau}
    &= -md\Big(\bm I_d-\frac{\check{\bm w}_{i, \tau} \check{\bm w}_{i, \tau}^\top}{d}\Big) \nabla_{\check{\bm w}_{i, \tau}}\mathcal L (\check {\bm a}_\tau,\check{\bm W}_{i})\\
    &= -\frac{\check a_{i, \tau}}{m}\sum_{j=1}^k \check a_{j, \tau}
    Y'(\langle \check{\bm w}_{i, \tau},\check{\bm w}_{j, \tau}\rangle / d)\cdot
    \bigl(\check{\bm w}_{j, \tau}-\langle \check{\bm w}_{i, \tau},\check{\bm w}_{j, \tau}\rangle \check{\bm w}_{i, \tau}\bigr) \\
    &+ \check a_{i, \tau} S'(\langle \bm w_\star,\check{\bm w}_{i, \tau}\rangle / d)\cdot
    \bigl(\bm w_\star-\langle \check{\bm w}_\star,\check{\bm w}_{i, \tau}\rangle \check{\bm w}_{i, \tau}\bigr).
\end{align}
To simplify the form, we define the alignments $\check R_{i, \tau},  \check Q_{ij, \tau}$ as 
\begin{align}
\check R_{i, \tau}=\frac{\langle \check{\bm w}_\star,\check{\bm w}_{i, \tau}\rangle}{d},
\quad
\check Q_{ij, \tau}=\frac{\langle \check{\bm w}_{i, \tau},\check{\bm w}_{j, \tau}\rangle}{d}.
\end{align}
We can derive ODEs for these order parameters.
\begin{align}\label{eq:ODE_population}
    \begin{aligned}        
        &\frac{d \check a_{i, \tau}}{d\tau}
        = -\frac1m\sum_{j=1}^m \check a_{j, \tau} Y(\check Q_{ij, \tau}) + S(\check R_{i, \tau}), \\
        &\frac{d \check R_{i, \tau}}{d\tau}
        = -\frac{\check a_{i, \tau}}{m}\sum_{j=1}^m \check a_{j, \tau} Y'(\check Q_{ij, \tau})(\check R_{j, \tau}-\check Q_{ij, \tau}\check R_{i, \tau})
        + \check a_{i, \tau} S'(\check R_{i, \tau})(1-\check R_{i, \tau}^2), \\
        &\frac{d \check Q_{ij, \tau}}{d\tau}
        = -\frac{\check a_{i, \tau}}{m}\sum_{k=1}^m \check a_{k, \tau} Y'(\check Q_{ik, \tau})(\check Q_{jk, \tau}-\check Q_{ik, \tau}\check Q_{ij, \tau})
        -\frac{\check a_{j, \tau}}{m}\sum_{k=1}^m \check a_{k, \tau} Y'(\check Q_{jk, \tau})(\check Q_{ik, \tau}-\check Q_{jk, \tau}\check Q_{ij, \tau})  \\
        &\qquad \quad + \check a_{i, \tau} S'(\check R_{i, \tau})(\check R_{j, \tau}-\check R_{i, \tau} \check Q_{ij, \tau})
        + \check a_{j, \tau} S'(\check R_{j, \tau})(\check R_{i, \tau}-\check R_{j, \tau} \check Q_{ij, \tau}).
    \end{aligned}
\end{align}
With the symmetric initialization, the following holds:
\begin{align}
\check a_{i, \tau}=\check a_\tau,\quad \check R_{i, \tau}=\check R_\tau,\quad \check Q_{ij, \tau}=\check Q_\tau (i\neq j).
\end{align}
Then, the ODE \eqref{eq:ODE_population} can be further simplified as follows when $m\to \infty$:
\begin{align}   
    &\frac{d\check a_\tau}{d\tau}=S(\check R_\tau)-aY(\check Q_\tau), \\
    &\frac{d\check R_\tau}{d\tau}
    = \check a_\tau(1-\check R_\tau^2)S'(\check R_\tau)-\check a_\tau^2(1-\check Q_\tau)\check R_\tau Y'(\check Q_\tau), \\   
    &\frac{d\check Q_\tau}{d\tau}
    = 2\bigl(\check a_\tau(1-\check Q_\tau)\check R_\tau S'(\check R_\tau)-\check a_\tau^2(1-\check Q_\tau)\check Q_\tau Y'(\check Q_\tau)\bigr).
\end{align}
Since we have 
\begin{align}
    \frac{d}{d\tau} (\check Q_\tau - \check R_\tau^2) 
    &= \frac{d\check Q_\tau}{d\tau} - 2 \check R_\tau \frac{d\check R_\tau}{d\tau} \\
    &= 2\bigl(\check a_\tau(1-\check Q_\tau)\check R_\tau S'(\check R_\tau)-\check a_\tau^2(1-\check Q_\tau)\check Q_\tau Y'(\check Q_\tau)\bigr) \\
    &\quad- 2\check R_\tau\big( \check a_\tau(1-\check R_\tau^2)S'(\check R_\tau)-\check a_\tau^2(1-\check Q_\tau)\check R_\tau Y'(\check Q_\tau)\big) \\
    &= -2\check a_\tau \check R_\tau S'(\check R_\tau) (\check Q_\tau - \check R_\tau^2) -2 \check a_\tau^2 (1 - \check Q_\tau) Y'(\check Q_\tau) (\check Q_\tau - \check R_\tau^2),
\end{align}
we obtain $\check Q_\tau=\check R_\tau^2$. With defining $T(\check R_\tau)=U(\check R_\tau^2)=\sum_{k=1}^\infty c_k^2 \check R_\tau^{2k}$, from $\check R_\tau U'(\check Q_\tau)=\tfrac12 T'(\check R_\tau)$, we obtain
\begin{align}   
    \frac{d\check R_\tau}{d\tau} &=
    \check a_\tau(1-\check R_\tau^2)S'(\check R_\tau)-\frac12 \check a_\tau^2(1-\check R_\tau^2)T'(\check R_\tau), \\
    \frac{d\check a_\tau}{dt} &= S(\check R_\tau)-\check a_\tau T(\check R_\tau).
\end{align}
This exactly matches the ODE \eqref{eq:original}.

\section{Origin of the fast-slow dynamics}\label{sec:origin_fast_slow}
The fast-slow structure observed in the dynamics of \eqref{eq:original} is
primarily motivated by numerical experiments, but it can be partially justified
theoretically in specific regimes.
In particular, the flow initially evolves purely in the $R$-direction at
$R=0$, since the $a$-component of the vector field vanishes there.
Moreover, when $|R|$ is small, and the trajectory evolves near the nontrivial
branch of the critical manifold, the Jacobian exhibits a strong separation of
eigenvalues, with the fast eigendirection nearly aligned with the $R$-axis.
In this regime, the fast-slow ansatz adopted in the main text is therefore
theoretically justified, which is especially relevant for the feature
unlearning scenarios studied in this work.

\subsection{Initial transient and quasi-frozen $a_\tau$.}
We consider the two-dimensional ODE \eqref{eq:original}. Then, from $S(0)=0,  T(0)=0$,  it holds that
\begin{equation}
g(0, \bar a)=S(0)- \bar a T(0)=0 .
\label{eq:g0}
\end{equation}
Moreover, since $S'(0) = 2c_{\star, 1} c_1$ and $T'(0)=0$,  we obtain
\begin{equation}
f(0, \bar a)
=\frac12  \bar a\bigl(2S'(0)- \bar a T'(0)\bigr)
= \bar a c_{\star,1}c_1>0.
\label{eq:f0pos}
\end{equation}
Therefore, at the initial time $\tau = 0$, one has $\dot a_\tau=0$ while $\dot R_\tau>0$ holds.

Expanding the vector field for small $f(R,a)$ (with $a$  treated as $O(1)$ during this short transient), we obtain
\begin{equation}
f(R,a)=f(0,a)+O(R),
\quad
g(R,a)= g_R(0,a) R + O(R^2),
\end{equation}
and hence it holds that
\begin{equation}
\frac{da_\tau}{dR_\tau}
=\frac{g(R_\tau,a_\tau)}{f(R_\tau,a_\tau)}
=O(R_\tau).
\label{eq:da_dR_small}
\end{equation}
This shows that $a_\tau$ remains approximately frozen while $R_\tau$ moves rapidly away from $0$, resulting in a fast relaxation toward $\mathcal S$.

\subsection{Critical manifold and scale separation for small $|R|$}
In this section, we discuss that the fast dynamics primarily drive the development of $R_\tau$, and this development is directed toward the coastal manifold $\mathcal{S}$.
The critical manifold (or $R$-nullcline) is defined by $f(R,a)=0$ and consists of three branches: $a=0$, $R = \pm 1$, and the nontrivial branch $a=h(R)$ defined in Assumption \ref{ass:bar_a}.
We focus on the last one, which is relevant to the feature unlearning phenomenon observed in numerical experiments.

Recall that $k_0\ge 2$ denotes the smallest integer such that $c_{\star,k_0}c_{k_0}\neq0$, and $k_1\ge 2$ denotes the smallest integer such that $c_{k_1}^2\neq0$.
Although the correlation functions satisfy $S(R)=O(R)$ and $T(R)=O(R^2)$ as $R\to0$, their leading-order derivatives are governed by these minimal indices.
In particular, as $R\to0$, we obtain
\begin{equation}
S'(R)=\Theta(1),
\qquad
S''(R)=\Theta(R^{k_0-2}),
\qquad
T'(R)=\Theta(R),
\qquad
T''(R)=\Theta(1).
\label{eq:ST_derivs_k}
\end{equation}

Along the nontrivial critical manifold $a=h(R)=2S'(R)/T'(R)$, the estimates
\eqref{eq:ST_derivs_k} immediately imply
\begin{equation}
a=h(R)=\Theta(R^{-1}),
\qquad (R\to 0),
\label{eq:a_asymp_smallR}
\end{equation}
so that the amplitude $a$ diverges algebraically as $R\to0$.

To quantify the resulting time-scale separation, we consider the Jacobian of the ODE \eqref{eq:original} at $(R_\tau, a_\tau) = (R, h(R))$,
\begin{equation}
J(R,a)=
\begin{pmatrix}
f_R & f_a\\
g_R & g_a
\end{pmatrix},
\quad
g_a=-T(R),\quad g_R=S'(R)-aT'(R),\quad f_a=(1-R^2)g_R,
\label{eq:jac_entries_1}
\end{equation}
and the remaining entry is given by:
\begin{equation}
f_R=\frac12 a\Bigl[(-2R)\bigl(2S'(R)-aT'(R)\bigr)
+(1-R^2)\bigl(2S''(R)-aT''(R)\bigr)\Bigr].
\label{eq:jac_entries_2}
\end{equation}
On the critical manifold $a=h(R)$, the first term in \eqref{eq:jac_entries_2} vanishes identically, yielding
\begin{equation}
f_R=\frac12 a(1-R^2)\bigl(2S''(R)-aT''(R)\bigr).
\label{eq:fR_on_manifold}
\end{equation}
Using \eqref{eq:ST_derivs_k} and \eqref{eq:a_asymp_smallR}, the dominant contribution arises from the $-a^2T''(R)$ term, so that we have
\begin{equation}
f_R
= \Theta(a^2)
= \Theta(R^{-2}),
\qquad (R\to 0).
\label{eq:fR_blowup}
\end{equation}

In contrast, the remaining Jacobian entries scale as
\begin{equation}
g_a=-T(R)= \Theta(R^2),
\qquad
g_R=S'(R)-aT'(R)= \Theta(1),
\qquad
f_a=(1-R^2)g_R= \Theta(1).
\label{eq:other_entries}
\end{equation}
Therefore, along the nontrivial critical manifold, the Jacobian has the schematic structure
\begin{equation}
J(R,h(R))
=
\begin{pmatrix}
\Theta(R^{-2}) & \Theta(1)\\
\Theta(1) & \Theta(R^{2})
\end{pmatrix},
\qquad (R\to 0).
\label{eq:J_scaling}
\end{equation}
Treating the large entry $f_R$ as dominant, the eigenvalues satisfy
\begin{equation}
\lambda_f = f_R + O(1)= \Theta(R^{-2}),
\qquad
\lambda_s = g_a - \frac{f_a g_R}{f_R} + O(R^2)
= O(R^{2}).
\label{eq:eig_asymp}
\end{equation}
Thus, the local time-scale ratio obeys
\begin{equation}
\frac{|\lambda_s|}{|\lambda_f|}
=O(R^{4})
\ll 1,
\qquad (R\to 0).
\label{eq:eig_gap}
\end{equation}
This establishes a pronounced scale separation of two eigenvalues along the critical manifold for sufficiently small $|R|$, especially when feature unlearning occurs.

We also discuss the fast eigenvector alignment with the $R$-direction.
Let $\bm v_f=(v_{f,R},v_{f,a})^\top$ denote the eigenvector associated with $\lambda_f$. Writing the eigenvector equation
\[
(f_R-\lambda_f)v_{f,R}+f_a v_{f,a}=0,
\qquad
g_R v_{f,R}+(g_a-\lambda_f)v_{f,a}=0,
\]
and using $\lambda_f\approx f_R$ gives the following from the second equation:
\begin{equation}
\frac{v_{f,a}}{v_{f,R}}
= -\frac{g_R}{g_a-\lambda_f}
= \Theta \left(\frac{1}{|f_R|}\right)
= \Theta(R^2),
\qquad (R\to 0),
\label{eq:vec_align}
\end{equation}
since $g_R=\Theta(1)$ and $g_a-\lambda_f=\Theta(R^{-2})$ by \eqref{eq:fR_blowup}, \eqref{eq:other_entries}, and \eqref{eq:eig_asymp}. Thus, the fast eigendirection satisfies that $\bm v_f $ is almost pararelly to $ \bm e_R$ up to $\Theta(R^2)$,  justifying the interpretation that $R$ is the fast variable near the critical manifold for small $|R|$.

\section{Proof of Section \ref{sec:theory}} \label{sec:proof_main_theorem}
The proof of the main theorem is based on \citet{lobry1998tykhonov}. It proceeds as follows:
\begin{itemize}
    \item[(1)] verify (H1) - (H5) of \citet{lobry1998tykhonov} hold for the model \eqref{eq:singular};

    \item[(2)] apply Theorem 1 of \citet{lobry1998tykhonov} to the system.  
\end{itemize}

In preparation, we introduce several notations related to the critical manifold $\mathcal{S}$.
Specifically, $\mathcal{S}$ can be decomposed as
$\mathcal{S}=\mathcal{S}_0^+ \sqcup \mathcal{S}_0^- \sqcup\mathcal{S}_1$, where $\mathcal{S}_1 := \{ (R,a) \in \{-1,1\} \times \R \}$ and $\mathcal S_0 := \mathcal S_0^+ \sqcup \mathcal S_0^-$ with
\begin{align}
    &\mathcal{S}_0^+ := \{ (R,a) \in (0,1) \times \R  \mid 2S'(R)-aT'(R)=0 \},\\ &\mathcal{S}_0^- := \{  (R,a) \in (-1,0) \times \R  \mid 2S'(R)-aT'(R)=0 \}.
\end{align}
Since $\bar f(0, \bar a) = \bar a c_{\star, 1} c_1 / \Lambda_f > 0$ from Assumptions \ref{ass:activation} and \ref{ass:symmetric_init}, we expect $R_{\tau_s}^\varepsilon$ rapidly increases and ride on $\mathcal S_0^+$ in the fast flow.

First, we show the following technical lemma.
\begin{lemma}\label{lem:h_alpha}
    For $h, \alpha : (0, 1) \to \R$ defined in Assumption \ref{ass:bar_a}, the following holds.
    \begin{align}
        &h(R) = \frac{c_{\star, 1}}{c_1} R^{-1} + o(R^{-1}), \\
        &\alpha(R)= -2 (k_0 - 1) k_0! c_{\star, k_0} c_{k_0} c_1^2 R^{k_0 + 1} + 2 (k_1 - 1) k_1! c_{\star, 1} c_1 c_{k_1}^2 R^{2 k_1} + O(R^{\max \{k_0 + 1, 2 k_1\} + 1})
    \end{align}
    as $R \to +0$.
\end{lemma}
\begin{proof}
    These follow from direct expansion:
    \begin{align}
        h(R) &= \frac{2S'(R)}{T'(R)} \\
        &= \frac{2 c_{\star, 1} c_1 + O(R)}{2 c_1^2 R + O(R^2)} \\
        &= \frac{c_{\star, 1}}{c_1} R^{-1} + o(R^{-1}), \quad R\to +0,
    \end{align}
    and 
    \begin{align}
        \alpha(R) &= S(R) T'(R) - 2 S'(R) T(R) \\
        &= (c_{\star, 1} c_1 R + k_0 ! c_{\star, k_0} c_{k_0} R^{k_0} + O(R^{k_0 + 1}))  (2c_1^2 R + 2 k_1\cdot k_1 ! c_{k_1}^2 R^{2 k_1 - 1} + O(R^{2 k_1})) \\
        &\quad -2 (c_{\star, 1} c_1 + k_0 \cdot k_0 ! c_{\star, k_0} c_{k_0} R^{k_0 -1} + O(R^{k_0}))
        (c_1^2 R^2 + k_1 ! c_{k_1}^2 R^{2k_1} + O(R^{2k_1 + 1})) \\
        &= -2 (k_0 - 1) k_0! c_{\star, k_0} c_{k_0} c_1^2 R^{k_0 + 1} + 2 (k_1 - 1) k_1! c_{\star, 1} c_1 c_{k_1}^2 R^{2 k_1} + O(R^{\max \{k_0 + 1, 2 k_1\} + 1}),
    \end{align}
    as $ R \to +0$.
\end{proof}
Now we state the following lemma, which states that Theorem 1 of \citet{lobry1998tykhonov} can be applied to the system \eqref{eq:singular}. In the following, we introduce $(\hat{R}_{\tau_f})_{\tau_f \geq 0}$ and $(\hat{a}_{\tau_s})_{\tau_s \geq 0}$ as a solution of a differential equation, then study its dynamics.
\begin{lemma}\label{lem:hypotheses}
    Under Assumptions \ref{ass:link}-\ref{ass:bar_a},
    for an open set $D = (-1, 1)\times (0, \infty) \in \mathbb{R}^2$ and any $M > \bar a$, all  the following hypotheses (H1) - (H5) hold. 
    \begin{enumerate}
        \item[(H1)]For any fixed $a\in(0, \infty)$, the fast equation
        \begin{align} \label{eq:fast}
            \frac{d\hat{R}_{\tau_f}}{d\tau_f}
            =\bar{f}(\hat{R}_{\tau_f}, a)\quad 
            \tau_f = \tau_s / \varepsilon
        \end{align}
        has a unique solution $(\hat{R}_{\tau_f})_{\tau_f}$ with prescribed initial conditions.

        \item[(H2)] There exits some $\delta > 0$  such that, for $I_a = [\bar{a} -\delta, M]$, there exits some function $\xi:I_a\to\mathbb{R}$ such that for any $a \in I_a$, $R = \xi(a)$ is an isolated root of an equation $\bar{f}(R, a) = 0$ and $\mathcal L := \{ (\xi(a), a); a \in I_a \} \in D$ holds.

        \item[(H3)]For any $a\in I_a$, $R = \xi(a)$ is an asymptotically stable equilibrium point of the fast equation, and we can take the basin of attraction of $R = \xi(a)$ uniformly over $I_a$.

        \item[(H4)] The slow equation  
        \begin{align} \label{eq:slow}
            \frac{d\hat{a}_{\tau_s}}{d\tau_s} = \bar{g}(\xi(\hat{a}_{\tau_s}), \hat{a}_{\tau_s})
        \end{align}
        defined on $\mathring I_a = \mathrm{int}\ I_a$ has a unique solution $(\hat{a}_{\tau_s})_{\tau_s}$ with prescribed initial conditions.

        \item[(H5)] $\bar{a}\in \mathring I_a$ holds, and the point $R(0) = 0$ is in the basin of attraction of  the equilibrium point $R = \xi(\bar{a})$ in the fast equation \eqref{eq:fast}.
    \end{enumerate}
\end{lemma}

\begin{proof}

\noindent{\bf (H1)}
For fixed $a$, $\bar{f}(R,a)$ is a polynomial of $R$, and especially, $C^\infty$ in $D$. So, (H1) follows from the Picard-Lindel\"of theorem.

\medskip
\noindent{\bf (H2)}
From Assumption \ref{ass:bar_a} and Lemma \ref{lem:h_alpha}, if we take $\delta>0$ sufficiently small, $I_a := [\bar{a} - \delta, M] \subset h((0, R^\star))$ holds.
Since $1-R^2>0$ and $T'(R) \ne 0$ when $R\ne0$, we have
\[
    \bar{f}(R,a)=0
    \quad\Longleftrightarrow\quad
    2S'(R)  - T'(R) a = 0
    \quad\Longleftrightarrow\quad
    a = h(R)=\frac{2S'(R)}{T'(R)}.
\]
From $I_a \subset h((0, R^\star))$ and Lemma \ref{lem:h_alpha}, there exists some closed interval $I_R \subset (0, R^\star)$ such that $h:I_R \to I_a$ is monotonically decreasing, and therefore, bijective.
Thus, the inverse function $\xi := h^{-1}: I_a \to I_R$ exists. Also, by its definition, $\bar{f}(\xi(a), a) = 0$ holds for $a\in I_a$, and we obtain $\mathcal L = \{ (R, a); \bar{f}(\xi(a), a) = 0,   a \in I_a \} \subset D$ from $I_R \subset (0, R^\star) \subset (-1, 1)$ and $I_a = [\bar a - \delta, M) \subset (0, \infty)$.

\medskip
\noindent{\bf (H3)}
We obtain that, for $R=\xi(a),   a\in I_a$
\begin{align}
    \partial_R \bar{f}(R,a)\left |_{R=\xi(a)} \right.
    &= \frac{1}{2 \Lambda_f} a (1 - R^2) \{ 2 S''(R) - a T''(R) \}   \\
    &= \frac{a (1 - R^2)}{\Lambda_f} \cdot \frac{S''(R) T'(R) - S'(R) T''(R)}{T'(R)} \\
    &= \frac{1}{\Lambda_f}S'(R) (1 - R^2) h'(R) \\
    &= \frac{1}{2\Lambda_f} (1 - R^2) T'(R) h(R) h'(R) \\
    &= \frac{1}{2\Lambda_f} (1 - \xi(a)^2) T'(\xi(a)) h(\xi(a)) h'(\xi(a)).
\end{align}
From $I_a \subset h((0, R^\star))$, we have $\xi(a) \in (0, R^\star)$ for any $a \in I_a$. Then, it follows that
\[
    1-\xi(a)^2 > 0,\quad
    T'(\xi(a)) > 0, \quad
    h(\xi(a))>0,\quad h'(\xi(a))<0.
\]
Therefore, $\partial_R \bar f(R,a) \left |_{R = \xi(a)} \right. < 0$ holds for any $a\in I_a$.
Hence, for any fixed $a\in I_a$,
$R=\xi(a)$ is an asymptotically stable equilibrium point of the fast equation.
Now, from the compactness of $I_a$, there exits some $\theta > 0$, such that, for any $a\in I_a$ and $R \in (\xi(a) - \theta, \xi(a) + \theta)$, $ \partial_R \bar{f}(R, a) < 0$ holds.
This means that we can take the basin of attraction of $R=\xi(a)$ uniformly over $I_a$.
 
\medskip
\noindent{\bf (H4)}
We can prove $\xi:\mathring I_a\to\mathring I_R$ is $C^\infty$ from the inverse function theorem. Then, the function $a \mapsto \bar{g}(\xi(a),a)$ is also $C^\infty$ in $\mathring I_a$. Hence, the slow equation \eqref{eq:slow}
has a unique solution from the Picard-Lindelöf theorem.

\medskip
\noindent{\bf (H5)}
By the definition of $I_a$, $\bar{a}\in \mathring I_a$.
Also, for any $a\in \mathring I_a$,  $\bar{f}(R, a) > 0$ for $0\le R<\xi(a)$, and $\bar{f}(\xi(a), a) = 0$. This means  $R=0$ is in the basin of attraction of $R=\xi(a)$.
\end{proof}
Together with Lemma \ref{lem:hypotheses} and Theorem 1 of \citet{lobry1998tykhonov}, we obtain the following theorem:
\begin{theorem}\label{thm:main2}
    Let each solution of the fast equation \eqref{eq:fast} with an initial value $0$, and the slow equation \eqref{eq:slow} with an initial value $\bar{a}$, be $\hat{R}_{\tau_f}$, $\hat{a}_{\tau_s}$, respectively. Let $T > 0$ be a maximal positive interval of definition of the slow equation \eqref{eq:slow}.
    Then, under Assumptions \ref{ass:link}-\ref{ass:bar_a}, the following holds;
    for any $\eta > 0$, there exists some $\epsilon_\eta > 0$ with the property that,  if $\epsilon < \epsilon_\eta$, the solution of \eqref{eq:singular} is defined for at least $\tau_s \in [0, T]$, and there exits $L > 0$ such that $\varepsilon L < \eta$, $| R_{\varepsilon \tau_f}^\varepsilon - \hat{R}_{\tau_f} | < \eta$ for $0 \le \tau_f \le L$, $|R_{\tau_s}^\varepsilon - \xi(\hat{a}_{\tau_s}) | < \eta$ for $ \varepsilon L \le \tau_s \le T$ and $|a_{\tau_s}^\varepsilon - \hat{a}_{\tau_s}| < \eta$ for $ 0 \le \tau_s \le T$. 
\end{theorem}

Next, based on the result above, we study the asymptotic behavior of $(\xi(\hat{a}_{\tau_s}), \hat{a}_{\tau_s})$ for $\tau_s \to\infty$ when the slow equation \eqref{eq:slow} is defined on $[\bar a - \delta, \infty)$ with $\delta > 0$ used in (H2) of Lemma \ref{lem:hypotheses}. We show the following lemma.
\begin{lemma}\label{lem:scaling_reduced}
    For the solution $\hat{a}_{\tau_s}$ of the slow equation \eqref{eq:slow} defined on $[\bar a - \delta, \infty)$ with $\delta > 0$ used in (H2) of Lemma \ref{lem:hypotheses}, 
    the following holds:
    \begin{align}\label{eq:reduced}
        \lim_{\tau_s\to\infty}\hat{a}_{\tau_s} = \infty, \quad
        \lim_{\tau_s\to\infty}\xi(\hat{a}_{\tau_s})= 0, \quad
        \lim_{\tau_s\to\infty}\hat{a}_{\tau_s}  \xi(\hat{a}_{\tau_s}) 
        = \frac{c_{\star, 1}}{c_1},
    \end{align}
    where the scaling law changes for each case of Assumption \ref{ass:alpha}:
    \begin{itemize}
        \item[(i)]  $\hat{a}_{\tau_s} = \Theta(\tau_s^{1 / 2k_1}), \quad
        \xi(\hat{a}_{\tau_s}) = \Theta(\tau_s^{-1 / 2k_1})$,
        
        \item[(ii)] $\hat{a}_{\tau_s} = \Theta(\tau_s^{1 / (k_0 + 1)}), \quad
        \xi(\hat{a}_{\tau_s}) = \Theta(\tau_s^{-1 / (k_0 + 1)})$,
    \end{itemize}
    when $\tau_s\to\infty$.
\end{lemma}
\begin{proof}
For $h:h^{-1}([\bar a - \delta, \infty)) \to [\bar a - \delta, \infty)$, from Lemma \ref{lem:h_alpha}, we have $R\to+0 \iff h(R) \to\infty$.
Thus, the inverse function $R = h^{-1}(a) = \xi(a)$ is
\begin{align}\label{eq:inverse}
    \xi(a) = \frac{c_{\star, 1}}{c_1}a^{-1} + o(a^{-1}), \quad a\to\infty.
\end{align}
We now consider the slow equation 
\begin{align}\label{eq:slow_detail}
\frac{d \hat a_{\tau_s}}{d\tau_s}
= \bar g(\xi(\hat a_{\tau_s}), \hat a_{\tau_s})
= \frac{\alpha(\xi(\hat a_{\tau_s}))}{\Lambda_s T'(\xi(\hat a_{\tau_s}))},
\quad  \hat a_{0} = \bar{a} >0.
\end{align}
Combining  Lemma \ref{lem:h_alpha}, \eqref{eq:inverse} and $T'(R) = 2c_1^2 R + O(R^2)$,  
in the case (i) of Assumption \ref{ass:alpha}, we have
\begin{align}
    \frac{d \hat a_{\tau_s}}{d\tau_s} &= \frac{1}{\Lambda_s} \cdot (k_1 - 1) k_1! c_{k_1}^2 \frac{c_{\star, 1}}{c_1} (c_{\star, 1} / c_1)^{2k_1}\hat a_{\tau_s}^{-2k_1 + 1} + o(\hat a_{\tau_s}^{-2k_1 + 1}) \\
    &= K \hat a_{\tau_s}^{-2k_1 + 1} + o(\hat a_{\tau_s}^{-2k_1 + 1}), \quad \hat a_{\tau_s} \to \infty,
\end{align}
for some constant $K>0$. From  \eqref{eq:slow_detail} and Assumption \ref{ass:bar_a}, $\hat a(\tau_s)$ monotonically increases along $\mathcal S$. Then, $\hat{a}_{\tau_s}$ is defined for $\tau_s \ge 0$, since $\hat{a}_{\tau_s} \in [\bar a - \delta, \infty)$ holds for $\tau_s \ge 0$. By standard comparison arguments for scalar ODEs, we have $\hat{a}_{\tau_s}\to\infty$ with $\hat{a}_{\tau_s} = \Theta(\tau_s^{1 / (2k_1)})$ when $\tau_s \to \infty$. From \eqref{eq:inverse}, $\xi(\hat{a}_{\tau_s})\to0$ with $\xi(\hat{a}_{\tau_s}) = \Theta(\tau_s^{-1/2k_1})$ also holds. The last equality of \eqref{eq:reduced} follows from $a   \xi(a) = c_{\star, 1} / c_1 + o(1), \quad a \to \infty$. We can derive another scaling law for (ii) in a similar way.
\end{proof}
Since we can take $T > 0$ arbitrarily large if we set $M > 0$ sufficiently large, we obtain the following statement.
\begin{corollary}\label{cor:epsilon_to_zero}
    $R_{\tau_s}^0,   a_{\tau_s}^0$ can be defined for any $\tau_s \in [0, \infty)$, and it holds that
    \begin{align}
        a_{\tau_s}^0  = \hat{a}_{\tau_s},
        \quad
        R_{
        \tau_s}^0 = \xi(\hat{a}_{\tau_s}),
    \end{align}
    for any $\tau_s > 0$.
\end{corollary}
\begin{proof}
    Fix $\tau_s > 0$. With sufficiently large $M$, we can set $T > \tau_s$. Then, from Theorem \ref{thm:main2}, for any $0 < \eta < \tau_s$, with sufficiently small $\varepsilon>0$, $|R_{\tau_s}^\varepsilon - \xi(\hat{a}_{\tau_s}) | < \eta$ holds. This directly implies $\lim_{\varepsilon\to+0}|R_{\tau_s}^\varepsilon - \xi(\hat{a}_{\tau_s}) | = 0$. We can prove similarly in the case of $a_{\tau_s}^\varepsilon$.
\end{proof}
Together with these results, we finally prove the main theorems.

\begin{proof}[Proof of Theorem \ref{thm:main}]
    From Lemma \ref{lem:scaling_reduced} and Corollary \ref{cor:epsilon_to_zero}, we obtain
    \begin{align}      \lim_{\tau_s\to\infty} a_{\tau_s}^0 = \lim_{\tau_s\to\infty}\hat{a}_{\tau_s} = \infty, \quad \lim_{\tau_s\to\infty}R_{\tau_s}^0 = \lim_{\tau_s\to\infty}\xi(\hat{a}_{\tau_s}) = 0,
    \end{align}
    and 
    \begin{align}       \lim_{\tau_s\to\infty} R_{\tau_s}^0 a_{\tau_s}^0 
    = \lim_{\tau_s\to\infty} \xi(\hat{a}_{\tau_s})  \hat{a}_{\tau_s}
    = \frac{c_{\star, 1}}{c_1}.
    \end{align}
\end{proof}
Quite similarly, Theorem \ref{thm:scaling} directly follows from the combination of Lemma \ref{lem:scaling_reduced} and Corollary \ref{cor:epsilon_to_zero}.

\section{Additional simulation} \label{sec:additional_simulation}

We perform numerical simulations of the ODE \eqref{eq:original} for multiple
choices of activation and link function coefficients $(c_k,c_{\star,k})$.
In particular, we consider the case with $\bar{k}_\star = \bar{k} = 7$ and fix a part of coefficients as $c_1 = c_2 = c_3 = 1$ and $ c_{\star, 1} = 1$. 
Then, we vary the rest of coefficients as $c_{2,\star}, c_{3,\star} \in \{-5, -1.67, 1.67, 5\}$.

Figure \ref{fig:dynamics_parameters} shows the results.
In all cases satisfying the condition of the initialization in Assumption \ref{ass:symmetric_init} identified in the theoretical analysis, we consistently observe that the dynamics of $(R_\tau, a_\tau)$ follows the result in Section \ref{sec:unlearning}.
Consequently, the result validates the overview of the feature unlearning along attracting branches of the critical manifold.

\begin{figure}[htbp]
  \centering
  \begin{minipage}{0.99\linewidth}
    \centering
    \includegraphics[width=\linewidth]{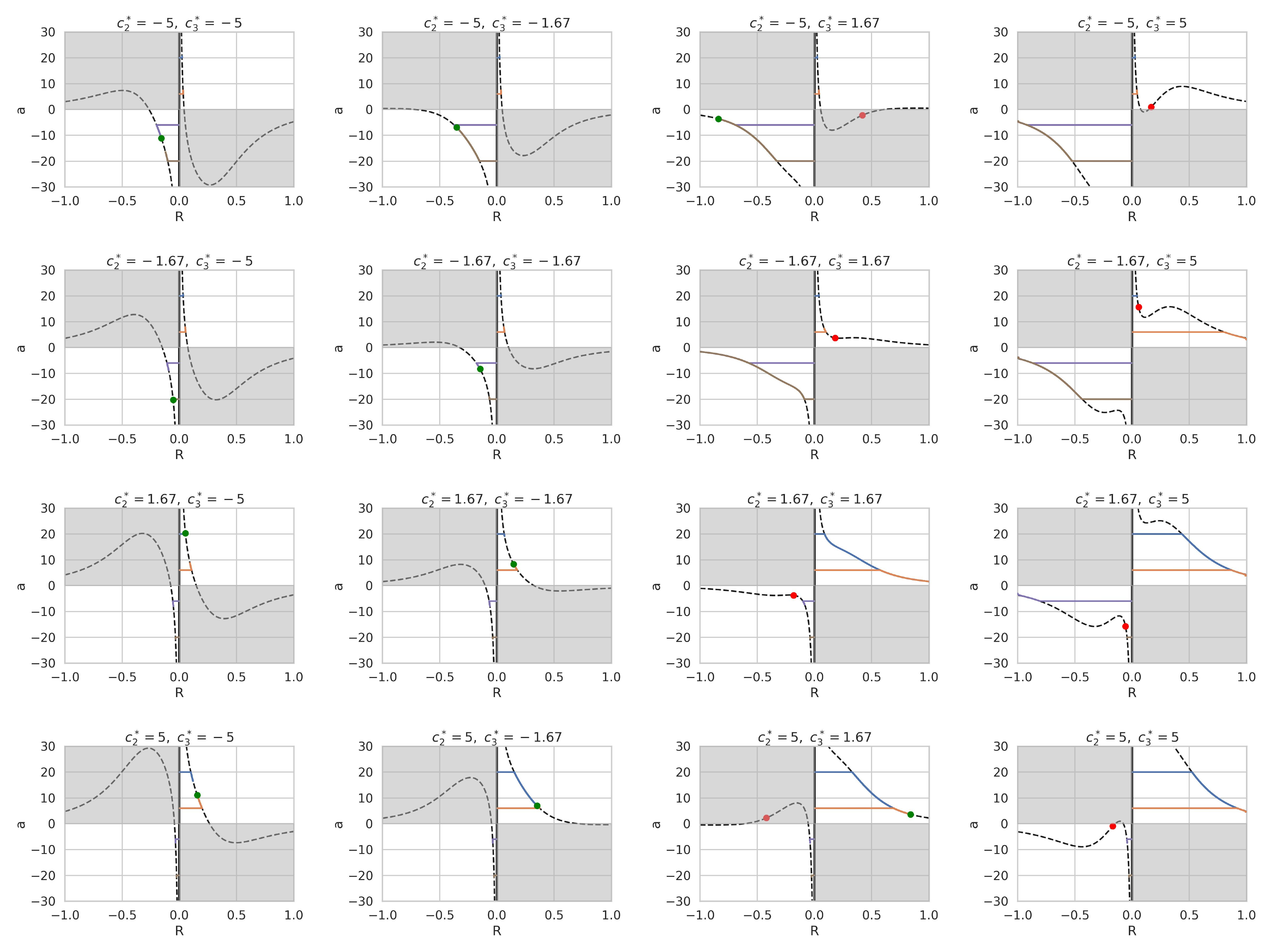}
  \end{minipage} 

  \caption{Dynamics of $(R_\tau, a_\tau)$ by the ODE \eqref{eq:original} with the various coefficients of the link function and the activation function. \label{fig:dynamics_parameters}}
\end{figure}

\bibliography{main}
\bibliographystyle{apecon}

\end{document}